\documentclass{article}
\usepackage[T1]{fontenc}
\usepackage{iclr2027_conference,times}

\usepackage{amsmath,amssymb}
\usepackage{graphicx}
\usepackage{xcolor}
\usepackage{hyperref}
\usepackage{url}

\definecolor{paperlink}{RGB}{21,74,138}
\hypersetup{colorlinks=true, linkcolor=paperlink, citecolor=paperlink,
            urlcolor=paperlink}

\usepackage{booktabs}


\usepackage{amsthm}
\makeatletter
\@ifundefined{theorem}{\newtheorem{theorem}{Theorem}}{}
\@ifundefined{proposition}{\newtheorem{proposition}{Proposition}}{}
\@ifundefined{corollary}{\newtheorem{corollary}{Corollary}}{}
\@ifundefined{lemma}{}{}
\theoremstyle{remark}
\@ifundefined{remark}{}{}
\theoremstyle{plain}
\makeatother

\newcommand{\TV}{\operatorname{TV}}
\newcommand{\Law}{\operatorname{Law}}

\newcommand{\Pref}{\mathbb{P}}               

\newcommand{\snr}{\gamma}                    
\newcommand{\Mimpl}{\widetilde{M}}           
\newcommand{\Mideal}{M^{\circ}}              

\usepackage{algorithm}
\usepackage{algpseudocode}

\newcommand{\result}[2]{\ensuremath{#1\,{\scriptstyle\pm\,#2}}}

\title{P3: Persistent Particle Planning\\
for Constrained Diffusion Control}

\author{Hikmet Simsir, Mahyar Fardinfar, Ozgur S. Oguz \\
Department of Computer Engineering, Bilkent University \\
\texttt{\{hikmet.simsir, mahyar.fardinfar, ozgur.oguz\}@bilkent.edu.tr}}

\iclrfinalcopy

\begin{document}

\maketitle

\begin{abstract}
  Diffusion models provide expressive priors over trajectories, but
  adapting these priors to test-time constraints requires maintaining
  feasibility and consistency across successive control decisions.
  We introduce Persistent Particle Planning (P3), a sequential Monte Carlo
  framework for diffusion control that maintains a weighted population
  of candidate plans across replanning steps.
  At each control step, P3 shifts and partially re-noises the candidate
  trajectories, refines them under the latest observation, and uses
  constraint-aware weighting and resampling to select among alternative
  continuations without retraining the diffusion model.
  We consider denoising and replanning as one Feynman--Kac particle system
  and analyze it under an idealized repair.
  We prove that the re-noising depth controls how reliably a kept plan
  stays on its route, and that keeping a rare, well-separated route takes
  far fewer plans than rediscovering it by sampling from scratch.
  Experiments under multiple test-time constraint configurations show
  that population reuse reduces route switching and improves success
  without constraint violations.
  Because P3 refines earlier plans instead of redrawing them, it also
  needs fewer denoising iterations per replan.
  On maze-navigation tasks, it plans faster than both regenerated
  populations and methods that correct a single sampled plan by
  constrained optimization.
  Code and pretrained models are available at
  \url{https://github.com/p3-username/p3-anon}.
\end{abstract}

\section{Introduction}
\label{sec:introduction}

Diffusion models represent multiple ways of completing a task as a
distribution over trajectories \citep{janner2022planning,chi2023diffusion}.
Requirements not included in the training data, such as an added obstacle,
must then be imposed at test time, typically by correcting each sample
during denoising, for instance by projecting it onto the feasible set
\citep{christopher2024constrained,romer2025diffusion}.
Such corrections act on the samples of a single planning call.

\begin{figure}[t]
    \centering
    \includegraphics[width=\linewidth]{
        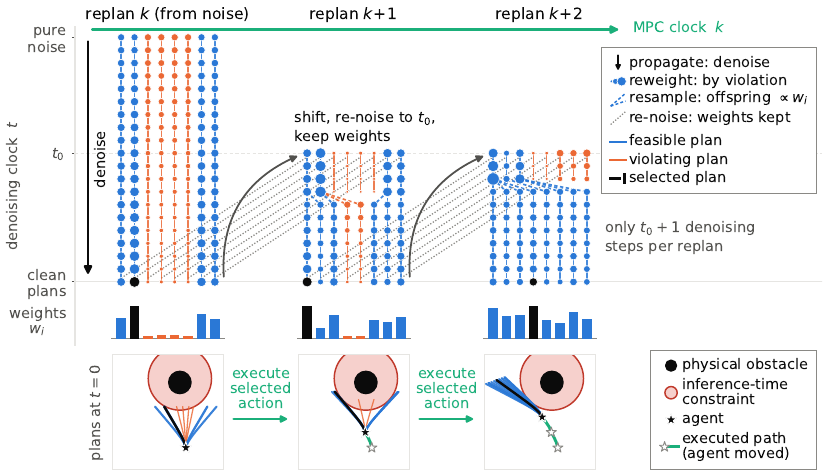
    }
    \vspace{-6pt}
    \caption{
        Persistent particle planning. Within each planning call (denoising
        clock, vertical), particles are denoised, weighted by their constraint
        violation and resampled. After the selected plan's first action is
        executed (control clock, horizontal), the population is shifted and
        re-noised to $t_0$ rather than redrawn from noise. Bottom: the
        corresponding plans around an obstacle and a test-time exclusion
        region.
    }
    \label{fig:persistent-planning}
    \vspace{-8pt}
\end{figure}

A controller, however, solves a sequence of planning problems.
In receding-horizon control, it executes the first action of a
finite-horizon plan and replans from the next observation
\citep{mayne2000constrained}, so consecutive problems share almost their
entire horizon.
Regenerating the plan from noise at every step discards that shared
solution and lets the controller alternate between modes, for example passing an obstacle on one side at
one step and on the other at the next.
Diffusion controllers reduce changes between successive plans by
ensembling or inpainting overlapping action chunks
\citep{zhao2023learning,black2025realtime}, selecting the sample most
coherent with the previous decision \citep{liu2025bidirectional}, or
warm-starting the next call from the previous plan by partial noising and
denoising, which Diffuser introduced to cut planning cost
\citep{janner2022planning} and which RDM \citep{zhou2023adaptive} and
streaming diffusion policies \citep{hoeg2025streaming} build into
replanning.
Warm-starting a single plan reuses the previous trajectory, but
constraint weighting cannot affect which plan is executed when there is
only one candidate.

We introduce \emph{persistent particle planning} (P3), which retains a
weighted population of plans within and across replanning calls
(Figure~\ref{fig:persistent-planning}).
Within each planning call, P3 runs sequential Monte Carlo (SMC) over the
denoising chain
\citep{gordon1993novel,delmoral2004feynman,chopin2020introduction}, as
prior work does for conditional and reward-guided diffusion
\citep{trippe2023diffusion,wu2023practical,kim2025test,singhal2025general}.
Particles follow the model's own denoising steps, and the constraints act
only by reweighting and resampling them according to the violations of
their predicted plans.
Across calls, P3 applies the shift-and-re-noise warm start to every
particle and carries the weights forward. The population and its
selection history become part of the controller's state.
To separate keeping plans from selecting among them, we run the two
population samplers, best-of-$N$ and SMC, both with and without carrying
the population.
Our contributions are:
\begin{itemize}
    \item \textbf{Method.} P3 carries a constraint-weighted population
    of diffusion plans across control steps. It updates each plan by
    shifting and partially re-noising it instead of redrawing it, and it
    can rejuvenate the population with new plans. It needs no retraining,
    and a warm call runs only $t_0+1$ of the $K$ denoising steps
    (Section~\ref{sec:method}).
    \item \textbf{Theory.} We consider denoising and replanning as a single
    Feynman--Kac model and analyze it under an idealized repair. The
    re-noising depth controls how reliably a kept plan stays on its
    route. A rare but well-separated route can be kept with far fewer
    plans than sampling from scratch needs to rediscover it. For independent
    acceptable plans drawn from the model, one exact repair at a fixed
    observation is at least as likely, on average, to leave an acceptable
    plan as drawing a fresh population of the same size
    (Section~\ref{sec:theory}).
    \item \textbf{Evidence.} We evaluate on obstacle-avoidance and
    maze-navigation benchmarks whose constraints are added only at test
    time. Persistence helps most where a feasible route is hard to find
    from scratch. There, P3 reaches the goal without violations far more
    often than a population of the same size regenerated at every step.
    Elsewhere, it matches or exceeds regenerated populations. Controls
    separate the two effects. Keeping plans reduces route switching, and
    constraint weighting and resampling further raise success without
    violations on several tasks. On maze navigation, P3 also plans faster than
    regenerating the population or correcting a single plan by
    optimization (Section~\ref{sec:experiments}).
\end{itemize}

\section{Related Work}
\label{sec:related}

\paragraph{Diffusion models for planning and control.}
In decision making, diffusion models have been used as trajectory
generators whose samples are steered toward rewards or goals by guidance
\citep{janner2022planning,ajay2023conditional}, as visuomotor policies
that denoise action chunks and execute them in a receding-horizon loop
\citep{chi2023diffusion}, and as motion planners that combine the
learned prior with collision costs \citep{carvalho2023motion}.
These models can represent multiple ways of completing a task
\citep{shafiullah2022behavior,jia2024towards}.
Model-based variants derive the score from a dynamics model rather
than from data \citep{pan2024model}.
P3 keeps the pretrained model and its transitions fixed and acts only on
which samples survive and which are reused.

\paragraph{Constraints at inference time.}
Test-time constraints are usually imposed by modifying the sampling
path.
Gradient guidance adds the gradient of a constraint cost to the score
\citep{dhariwal2021diffusion,janner2022planning}, control-barrier
guidance enforces finite-time invariance of the denoising dynamics
\citep{xiao2025safediffuser}, and projection methods map each iterate
or its clean prediction onto the feasible set
\citep{christopher2024constrained,romer2025diffusion}; DiRecT argues
that projecting noisy iterates over-constrains the sampler and instead
enforces the constraint on the final sample through a receding-horizon
optimal-control formulation \citep{direct2026}.
A separate line changes the model so that its state space is the
constrained domain \citep{lou2023reflected,fishman2023diffusion}.
These approaches modify the sampling dynamics or the samples themselves.
P3 leaves the proposal untouched and reweights its reference path law.
Unlike exact projection, reweighting does not place excluded mass on a
previously null constraint boundary; its accumulated path potentials
can still change the relative masses of feasible routes
(Appendices~\ref{app:proof-target} and~\ref{app:proof-transport}).

\paragraph{Sequential Monte Carlo for diffusion sampling.}
SMC approximations to Feynman--Kac models have been used for conditional
and reward-guided diffusion sampling.
SMC over the denoising chain conditions on partial observations
\citep{trippe2023diffusion}; twisted proposals with telescoped
potentials give asymptotically exact conditional samplers
\citep{wu2023practical}; and filtering formulations address linear
inverse problems \citep{cardoso2024monte,dou2024diffusion}.
Related SMC methods are used for inference-time reward alignment
\citep{li2024derivative,kim2025test,singhal2025general,skreta2025feynman},
diverse batch sampling through joint potentials \citep{corso2024particle},
and search over initial noises \citep{ma2025inference}.
The choice of potential affects sampling efficiency, and twisting
functions can be learned
\citep{whiteley2014twisted,guarniero2017iterated,heng2020controlled,zhao2024probabilistic,wang2026trust}.
These methods sample one target with a population created from noise
and discarded afterward.
P3 uses a plain plug-in potential, to which learned twisting could be
added, and keeps the population, its weights and its lineages across
replanning calls; for fixed observations this is one Feynman--Kac model
spanning denoising and replanning (Section~\ref{sec:method}).

\paragraph{Replanning, warm starts, and populations in control.}
Sampling-based model predictive control carries information across
control steps. MPPI and cross-entropy planners shift and perturb the
previous control sequence \citep{williams2017information,chua2018deep},
SMC solves the open-loop problem at each step
\citep{kantas2009sequential,piche2019probabilistic}, and Stein
variational MPC keeps control-sequence particles between steps to
preserve modes \citep{lambert2020stein}; particle filters recover from
a lost track by injecting fresh particles \citep{thrun2005probabilistic}
and diversify by invariant moves \citep{gilks2001following}.
Diffuser warm-starts each plan by partially noising and denoising the
previous one \citep{janner2022planning,meng2022sdedit}, RDM reuses one
plan until the model's likelihood triggers regeneration
\citep{zhou2023adaptive}, and Streaming Diffusion Policy rolls a
partially denoised trajectory forward \citep{hoeg2025streaming}; all
keep one trajectory.
Diffusion-ES mutates high-reward diffusion samples by partial noising
and denoising, and can warm-start its population with solutions from the
previous control step \citep{yang2024diffusion}. Guided flow-matching
samples have also been refined by MPPI with warm starts
\citep{mizuta2025unified}.
P3 weights and resamples particles during denoising, then carries the
plans and their weights across control steps. It repairs this population
by partial noising and denoising and can add fresh particles.

\section{Persistent Particle Planning}
\label{sec:method}

\subsection{Problem formulation}

We consider a pretrained diffusion model over finite-horizon plans
$x=(x_0,\ldots,x_{H-1})$ with transitions $x_h=(a_h,o_h)$, where $a_h$
is an action and $o_h$ the predicted observation (e.g., positions and
velocities).
At control step $k$, the controller conditions generation on the current
observation $s_k$, selects a plan, executes its first action, and replans
after the next observation.

Constraints are supplied at inference time without retraining the
diffusion model.
Each constraint $j$ contributes a residual $c_j(x)\geq 0$, evaluated in
physical coordinates, which is the positive violation of an inequality
or the absolute residual of an equality.
The violation of a plan aggregates the residuals,
\begin{equation}
    c(x)=A\bigl(c_1(x),\ldots,c_J(x)\bigr),
    \label{eq:plan-cost}
\end{equation}
with any aggregate $A$ that vanishes exactly when every residual does,
such as their maximum or their sum; the residuals and aggregate used for
each task are given in Appendices~\ref{app:impl}
and~\ref{app:protocol}.

\subsection{Constraint-weighted diffusion sampling}
\label{sec:particle-sampling}

A planning call maintains $N$ noisy plans $\{y_t^{(i)}\}_{i=1}^{N}$.
At each reverse iteration, the diffusion model predicts a clean plan
\begin{equation}
    \hat{x}_t^{(i)}
    =
    D_\theta(y_t^{(i)},t,s_k),
\end{equation}
whose initial observation coordinates are pinned to $s_k$.
We evaluate a constraint potential on each clean prediction,
\begin{equation}
    \psi_t^{(i)}
    =
    \exp\!\left(
        -\beta_t\lambda_t c(\hat{x}_t^{(i)})
    \right),
    \label{eq:potential}
\end{equation}
whose strength $\beta_t\lambda_t$ increases as denoising progresses
(Appendix~\ref{app:impl}).
The potentials accumulate in the particle weights,
\begin{equation}
    \widetilde{w}_t^{(i)}
    =
    w_{\mathrm{in}}^{(i)}\psi_t^{(i)},
    \qquad
    W_t^{(i)}
    =
    \frac{\widetilde{w}_t^{(i)}}
         {\sum_j\widetilde{w}_t^{(j)}},
    \label{eq:weight-update}
\end{equation}
where $w_{\mathrm{in}}^{(i)}$ is the weight entering the iteration.
When the effective sample size $1/\sum_i (W_t^{(i)})^2$ falls below
$N/2$, the population is resampled according to $W_t$ and the weights
are reset to uniform; otherwise the normalized weights are retained.
The particles then propagate through the learned reverse diffusion
transition.

This procedure favors lineages whose predicted clean trajectories
satisfy the constraints during denoising.
Because the proposal is the original diffusion transition, no
proposal-density correction is required.
Accumulating intermediate potentials defines a path-dependent
selection objective that is not equivalent to weighting only the final
clean trajectory (Appendix~\ref{rem:marginal}, paragraph
``Current-plan marginal'').

\subsection{Reusing the population across control steps}
\label{sec:persistence}

The first planning call initializes particles independently from
Gaussian noise and runs the complete reverse process.
After executing an action, the controller advances every retained plan
by one transition:
\begin{equation}
    S(x)=
    \bigl(x_1,\ldots,x_{H-1},x_{H-1}\bigr).
    \label{eq:shift}
\end{equation}
Duplicating the final transition restores the planning horizon.
The next denoising call repairs this extended tail together with the
rest of the trajectory.

Each shifted plan is partially re-noised:
\begin{equation}
    y_{t_0}^{(i)}
    =
    \sqrt{\bar{\alpha}_{t_0}}\,
    S(x^{(i)})
    +
    \sigma_{\mathrm{init}}
    \sqrt{1-\bar{\alpha}_{t_0}}\,
    \epsilon^{(i)},
    \qquad
    \epsilon^{(i)}\sim\mathcal{N}(0,I).
    \label{eq:renoise}
\end{equation}
Here $x^{(i)}$ is the final plan of particle $i$ in the previous call,
$\bar\alpha_t$ is the noise schedule's cumulative signal coefficient, which
decreases from near $1$ at $t=0$ toward $0$ at the noisiest level, and
$\sigma_{\mathrm{init}}$ is the noise scale of the checkpoint's sampler
(Appendix~\ref{app:impl}).
The first observation of each re-noised plan is then replaced by the
new measurement $s_k$, as in every denoising step, and denoising
continues from level $t_0$.

The re-noising level $t_0$ controls both trajectory reuse and
computation. A small value retains more of the previous plan and
requires fewer denoising iterations, whereas a large value perturbs the
plan more and allows more iterations for adaptation.
A warm call runs $t=t_0,\ldots,0$, that is, $t_0+1$ of the $K$ reverse
iterations, on the original potential schedule.
Particle weights are retained across replanning boundaries in the main
configuration, so both the candidate trajectories and their selection
history persist; resetting them is evaluated as an ablation.

After denoising, the controller executes the first action of the
minimum-violation plan, $i^\star=\arg\min_i c(x^{(i)})$, and retains
the entire population for the next control step.
Algorithm~\ref{alg:p3} in Appendix~\ref{app:impl} gives the full
procedure.

\subsection{Theoretical analysis}
\label{sec:theory}

\paragraph{Particle target.}
\label{sec:fk}
Fix the observations and index denoising and shift-and-re-noise events
by $r=0,\ldots,R$. Write $Y_r$ for the noisy plan at event $r$,
with diffusion level $t(r)$, and $Y_{0:R}$ for its history.
$Y_0$ is the initial Gaussian-noise plan with the first observation
imposed. Let $\Pref$ be the law of this history without constraint
weighting or resampling.
With $\Lambda_t=\beta_t\lambda_t$, the path target is
\begin{equation}
    \Gamma_R(dY_{0:R})=\Pref(dY_{0:R})
        \prod_{r=0}^{R}e^{-\Lambda_{t(r)}
                            c(\hat x_{t(r)}(Y_r))},
    \qquad \Pi_R=\Gamma_R/\Gamma_R(1).
    \label{eq:fk}
\end{equation}
P3 is the SMC particle system for this model;
carried weights retain a lineage's selection history across calls
(Proposition~\ref{prop:target}, Appendix~\ref{app:proof-target}).
With $N=1$, constraint weighting cannot change the executed plan.

On cluttered, one of the obstacle-avoidance configurations of
Section~\ref{sec:experiments}, the fraction of trials that reach the goal
without violations is $0.35$ with $16$ fresh plans, $0.06$ with one
retained plan, and $0.99$ with $16$ retained plans
(Figure~\ref{fig:theory}(a)).
Under an idealized, exact repair (assumption (E) below), adding more
noise before denoising lowers the probability that a retained plan
remains on its route. Proofs and diagnostics are in Appendix~\ref{app:proofs}.

\begin{figure}[t]
    \centering
    \includegraphics[width=\linewidth]{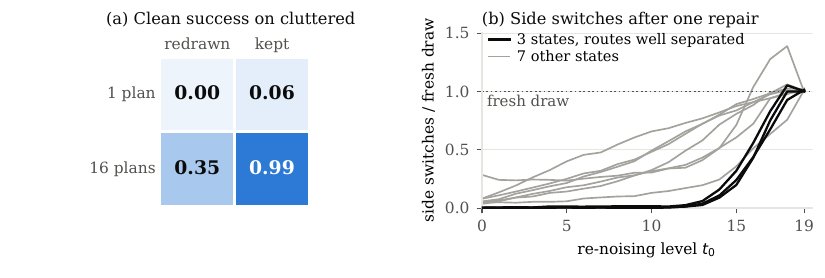}
    \vspace{-6pt}
    \caption{
        (a) Clean success on the cluttered configuration of
        Section~\ref{sec:experiments}, with one or $16$ plans, redrawn or
        kept (Table~\ref{tab:abl-pop}).
        (b) Side switches after one learned repair, relative to a fresh
        draw, at ten fixed Avoiding states where fresh plans end on both
        sides of the centerline; black: the three states whose two routes
        are well separated.
    }
    \label{fig:theory}
    \vspace{-8pt}
\end{figure}

Imagine drawing a route around an obstacle and then partly erasing it.
If enough remains to tell which side it takes, filling it back in tends
to preserve that choice. As more is erased, reconstruction becomes more
like choosing a route from scratch. Theorem~\ref{thm:memory} shows that adding
more noise makes repair less likely to preserve the original side and
more likely to switch sides.

\paragraph{Exact repair (E).}
At a fixed observation, let $p$ be the model's plan law.
Re-noise $x\sim p$ without shifting, then draw $x'\sim p(\cdot\mid y)$
independently of $x$ given $y$.
Rescaling eq.~\eqref{eq:renoise} gives
$y=\snr_{t_0}x+\epsilon$, $\epsilon\sim\mathcal{N}(0,I)$, where
$\snr_t=\sigma_{\mathrm{init}}^{-1}
\sqrt{\bar\alpha_t/(1-\bar\alpha_t)}$ decreases with $t$.

\begin{theorem}[Retention under re-noising]
\label{thm:memory}
Under (E), for any measurable $A$ with $\nu=p(A)\in(0,1)$, write
$V_A=\operatorname{Var}[p(A\mid y)]$.
(i) The retention and boundary-crossing probabilities are
\begin{equation}
    \Pr(x'\in A\mid x\in A)=\nu+\frac{V_A}{\nu},
    \qquad
    \Pr(\mathbf{1}_A(x')\ne\mathbf{1}_A(x))
    =2\nu(1-\nu)-2V_A.
    \label{eq:memory-main}
\end{equation}
(ii) $V_A$ decreases strictly with $t_0$ and tends to zero as
$\snr_{t_0}\to0$, when the repair becomes a fresh draw from $p$.
\end{theorem}

The covariance identity in (i) is classical data augmentation covariance identity
\citep{liu1994covariance}.
Here, $V_A$ measures how much the conditional route probability
$p(A\mid y)$ varies across noisy plans; retention exceeds the
fresh-draw probability by $V_A/\nu$.
At the three fixed states with well-separated routes, repair changes
the plan's endpoint side at most $3\%$ as often as independent redrawing
for $t_0\le12$. At $t_0=19$, the switch rates are approximately equal
(Figure~\ref{fig:theory}(b); Appendix~\ref{app:kernel}).

\begin{corollary}[Rare-route retention]
\label{cor:memory-routes}
Under (E), let $p$ mix $\mathcal{N}(\mu_a,s^2I)$ and
$\mathcal{N}(\mu_b,s^2I)$ with weights $0<\pi\le1/2$ and $1-\pi$.
Here, $s$ is the standard deviation in each plan coordinate within
either component, assumed equal for the two components.
If the original plan is sampled from component $a$, the repaired plan
also belongs to component $a$ with probability at least $1/2$ if and only if
\begin{equation}
    \kappa:=\frac{\|\mu_a-\mu_b\|}
    {\sqrt{s^2+\snr_{t_0}^{-2}}}
    \ge\sqrt{2\ln((1-\pi)/\pi)}.
    \label{eq:route-threshold}
\end{equation}
Finding at least one plan from component $a$ with probability at least
$1/2$ requires
$N\ge\ln2/\ln(1/(1-\pi))$ independent fresh draws.
\end{corollary}

A rarely chosen route can still be easy to recognize after adding noise
if it is sufficiently different from the alternatives. Repair can then
preserve it even though a fresh draw would be unlikely to find it. At $\pi=0.01$, the $1/2$-retention threshold is $\kappa\approx3.03$,
whereas finding at least one plan from that component with the same
probability takes $69$ independent fresh draws.
These are mixture labels; geometric route regions agree with them
when the components are well separated.

\paragraph{Independent plans (S).}
Condition on the pre-planning history $\mathcal G_k$, which fixes the
current observation and held plans. Final plans are independent. Each
held plan is repaired with independent randomness, and fresh plans are
drawn independently from the unweighted plan law. There is no
within-call weighting or resampling.

\begin{proposition}[Kept and fresh plans]
\label{prop:keep}
Let $A_k=\{z:c_k(z)\le\delta\}$, $\delta\ge0$, so the minimum-violation plan is
acceptable iff the final population contains one.
Under (S), $m$ repairs of held plans $z_i$ have acceptance probabilities
$r_k(z_i)$, and $n$ fresh draws have rate $\nu_k$.
Then:
\begin{enumerate}
    \item[(i)] The probability of no acceptable plan is
    \begin{equation}
        (1-\nu_k)^n\prod_{i=1}^{m}(1-r_k(z_i)).
        \label{eq:keep-formula}
    \end{equation}
    Thus $n$ fresh draws restore an acceptable plan with probability
    at least $1-(1-\nu_k)^n$, whatever is held.
    \item[(ii)] For $N$ fresh plans with mean marginal acceptance
    $\nu_k$, the probability of finding one is at most $N\nu_k$,
    without independence; for independent draws with a common rate
    it is $1-(1-\nu_k)^N$.
    \item[(iii)] Under (E), with $\nu,V_A$ as in
    Theorem~\ref{thm:memory}, if the $m$ held plans
    are i.i.d.\ from $p(\cdot\mid A)$, with independent repairs and
    $n$ independent fresh draws from $p$, the one-step failure
    probability is $(1-\nu-V_A/\nu)^m(1-\nu)^n\le(1-\nu)^{m+n}$.
\end{enumerate}
\end{proposition}

Under the assumptions in part (iii), keeping
and repairing acceptable plans is at least as likely to leave us with an
acceptable plan as drawing a fresh population of the same size.

\subsection{Rejuvenation}
\label{sec:rejuvenation}
Under (S), $n$ new plans can restore an acceptable plan whatever is held
(Proposition~\ref{prop:keep}(i)). This motivates rejuvenation.
At each boundary, P3 can replace its $n$ lowest-weight particles with new
plans, sampled from noise down to level $t_0+1$ by the same weighted
sampler. They enter at $t_0$ with the population's mean weight and cost
$n(K-1-t_0)$ extra evaluations per step.
Rejuvenation is a heuristic outside our analysis. Its new plans replace
weight-selected particles and are not independent draws
(Appendix~\ref{app:proofs}).

\section{Experiments}
\label{sec:experiments}

Section~\ref{sec:main-results} compares P3 with the baselines and reports
planning time.
Section~\ref{sec:ablations} separates keeping plans from selecting among
them and reports ablations.

\subsection{Setup}
\label{sec:setup}

\begin{figure}[t]
    \centering
    \includegraphics[width=\linewidth]{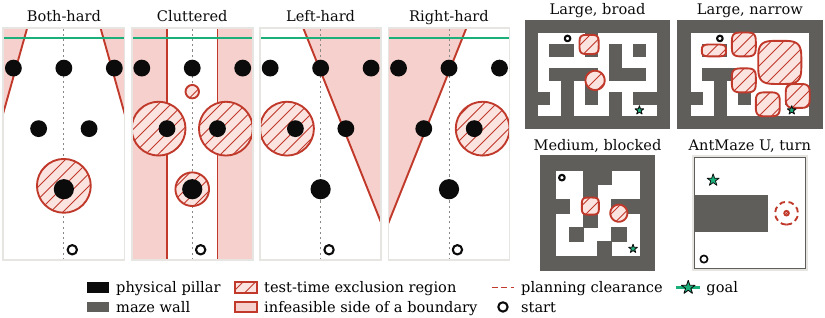}
    \vspace{-6pt}
    \caption{
        Tasks and test-time constraints. Left: D3IL Avoiding with its four
        added configurations; the dotted line is the centerline of the
        plan-side switch metric. Right: the large PointMaze with its broad and
        narrow presets (top), the medium PointMaze with our blocked preset, whose only
        clean route is the detour, and AntMaze U with an exclusion disc at the
        turn (bottom). None of the added constraints appears in the training
        data.
    }
    \label{fig:envs}
    \vspace{-8pt}
\end{figure}

We evaluate on four tasks (Figure~\ref{fig:envs}).
\textbf{Avoiding} is the D3IL obstacle-avoidance task
\citep{jia2024towards}, in which a robot end effector crosses a field of
six pillars to a goal line. The model plans over desired-position
increments, desired positions and measured positions.
Four configurations add static constraints (Figure~\ref{fig:envs}).
\emph{Both-hard}, \emph{left-hard} and \emph{right-hard} were introduced
by \citet{romer2025diffusion}, and \emph{cluttered}, a narrow corridor
with four exclusion regions, by \citet{jung2025joint}.
\textbf{PointMaze} \citep{delazcano2023gymnasium,fu2020d4rl} moves a
force-actuated point mass between fixed start and goal cells, and the
model plans over forces, positions and velocities.
On the large layout, the \emph{broad} and \emph{narrow} obstacle presets
clip the corridors of the unguided route; on the medium layout,
\emph{blocked} closes the direct corridor and leaves a detour that the
demonstrations contain but the unguided model rarely executes.
\textbf{AntMaze U} adds an exclusion disc at the turn of the U-shaped D4RL
AntMaze \citep{fu2020d4rl}, planned with a goal-conditioned state
diffusion model and an inverse-dynamics head.
Appendix~\ref{app:protocol} gives the shared protocol (horizons,
re-noising depths, margins and trial counts), and
Appendices~\ref{app:maze}, \ref{app:mb} and~\ref{app:antmaze} the maze
layouts, models and presets.


\paragraph{Baselines.}
\label{sec:baselines}
\emph{Unguided} draws one cold sample without constraint handling, and
\emph{Projection} projects one onto the constraints during denoising,
as in DPCC \citep{romer2025diffusion}.
\emph{RDM + projection} keeps one plan and repairs or regenerates it as
RDM \citep{zhou2023adaptive} does, projecting it onto the constraints
(Appendix~\ref{app:rdm}).
RDM as published does not use the added constraints.
\emph{DiRecT} \citep{direct2026} solves, in the second half of denoising,
a constrained optimization on the model's clean-plan estimate and
corrects the noisy plan by the result (Appendix~\ref{app:direct}).
\emph{SafeDiffuser} \citep{xiao2025safediffuser} applies a control-barrier
QP to the planned positions at every
denoising step (Appendix~\ref{app:safediffuser}).
\emph{Best-of-$N$} executes the minimum-violation plan among $N$ cold
samples, and \emph{SMC} (cold SMC) is the sampler of
Section~\ref{sec:particle-sampling} with a fresh population at every
step; P3 adds trajectory and weight reuse to it, and rejuvenation on
Avoiding. \emph{Persistent best-of-$N$} is P3 without weighting or
resampling (Appendix~\ref{app:pbon}).
Settings per task are in Appendix~\ref{app:protocol}.

\paragraph{Metrics.}
\emph{Clean success} is the fraction of trials that reach the goal
without any violation of the added constraints, graded on the unpadded
constraints; task success and the accumulated violation are in
Appendix~\ref{app:full-results}.
\emph{Plan switches} count how often the selected plan's endpoint
changes side of the workspace centerline on Avoiding, a route-consistency
proxy rather than a homotopy test.
Means and standard deviations are across model seeds (evaluation seeds
on AntMaze); bold marks the best mean.

\subsection{Main results}
\label{sec:main-results}

\begin{table}[t]
    \centering
    \footnotesize
    \caption{
        Clean success under test-time constraints, mean $\pm$ standard
        deviation across five model seeds (AntMaze: six evaluation seeds).
        Avoiding columns: both-hard, cluttered, left-hard, right-hard. Upper
        rows plan one sample per step, lower rows a population. Trial
        counts and settings: Appendix~\ref{app:protocol}; additional
        metrics: Appendix~\ref{app:full-results}.
    }
    \label{tab:results}
    \label{tab:avoiding-performance}
    \label{tab:maze-results}
    \resizebox{\linewidth}{!}{
\setlength{\tabcolsep}{2pt}
\begin{tabular}{@{}lcccccccc@{}}
\toprule
 & \multicolumn{4}{c}{Avoiding} & \multicolumn{2}{c}{Large maze} & Medium & AntMaze \\
\cmidrule(lr){2-5} \cmidrule(lr){6-7} \cmidrule(lr){8-8} \cmidrule(lr){9-9}
Method & Both & Clutt. & Left & Right & Broad & Narrow & Blocked & Turn \\
\midrule
Unguided & \result{0.00}{0.01} & \result{0.00}{0.00} & \result{0.09}{0.07} & \result{0.05}{0.08} & \result{0.06}{0.10} & \result{0.02}{0.05} & \result{0.00}{0.00} & \result{0.38}{0.11} \\
\shortstack[l]{RDM +\\projection} & \result{0.60}{0.11} & \result{0.09}{0.05} & \result{0.60}{0.10} & \result{0.36}{0.14} & \result{0.50}{0.11} & \result{0.72}{0.10} & \result{\mathbf{0.24}}{0.09} & \result{0.37}{0.15} \\
Projection & \result{0.85}{0.05} & \result{0.46}{0.26} & \result{0.83}{0.05} & \result{0.79}{0.10} & \result{0.66}{0.22} & \result{0.91}{0.15} & \result{0.11}{0.09} & \result{0.40}{0.04} \\
DiRecT & \result{0.92}{0.06} & \result{0.28}{0.12} & \result{0.84}{0.03} & \result{0.88}{0.10} & \result{\mathbf{1.00}}{0.00} & \result{\mathbf{0.98}}{0.04} & \result{0.12}{0.11} & \result{0.48}{0.09} \\
SafeDiffuser & \result{0.44}{0.17} & \result{0.16}{0.12} & \result{0.66}{0.31} & \result{0.34}{0.30} & \result{0.30}{0.36} & \result{0.04}{0.08} & \result{0.00}{0.01} & \result{0.39}{0.14} \\
\midrule
Best-of-$N$ & \result{0.92}{0.06} & \result{0.34}{0.35} & \result{0.88}{0.07} & \result{0.85}{0.19} & \result{0.74}{0.29} & \result{0.51}{0.38} & \result{0.01}{0.01} & \result{0.43}{0.11} \\
SMC & \result{0.97}{0.03} & \result{0.48}{0.41} & \result{0.91}{0.06} & \result{0.97}{0.04} & \result{0.75}{0.36} & \result{0.51}{0.39} & \result{0.00}{0.00} & \result{0.53}{0.11} \\
\shortstack[l]{Persistent\\best-of-$N$} & \result{0.93}{0.03} & \result{0.87}{0.14} & \result{0.86}{0.09} & \result{0.92}{0.13} & \result{0.75}{0.18} & \result{0.56}{0.36} & \result{0.20}{0.07} & \result{0.37}{0.06} \\
P3 (ours) & \result{\mathbf{0.99}}{0.01} & \result{\mathbf{0.99}}{0.01} & \result{\mathbf{0.94}}{0.04} & \result{\mathbf{0.98}}{0.04} & \result{0.89}{0.09} & \result{0.85}{0.21} & \result{0.21}{0.13} & \result{\mathbf{0.62}}{0.10} \\
\bottomrule
\end{tabular}
}
    \vspace{-8pt}
\end{table}

P3 attains the highest mean clean success in five of the eight columns
of Table~\ref{tab:results}. Its gains over populations regenerated at
every step are largest on cluttered, narrow and the medium maze, where a
clean route is hard to find from scratch.

\paragraph{Avoiding.}
On cluttered, clean success is $0.99$ for P3 and $0.48$ for SMC.
Single-plan correction methods reach only $0.46$ (Projection), $0.28$
(DiRecT) and $0.16$ (SafeDiffuser) there. An $8$-step plan does
not span its exclusion regions, so the correction can put consecutive
plans on different sides of a region.
Projection switches plan side $12.4$ times per episode there, against
$7.7$ for P3 (Appendix~\ref{app:full-results},
Table~\ref{tab:persistence-effects} and Figure~\ref{fig:route}).
On both-hard, left-hard and right-hard, P3 and SMC with the same $N$ and
margin differ in mean clean success by at most $0.03$.

At equal margins, P3 is cleaner than Projection on all four
configurations, and a larger margin helps Projection only on both-hard
(Appendix~\ref{app:full-results}, Table~\ref{tab:margin-sensitivity}).
RDM + projection is below Projection on all four Avoiding cases.
On cluttered, the projection step fails more often when RDM tries to
fix its old plan than when it starts with a new one
(Appendix~\ref{app:rdm}).

\paragraph{Mazes.}
On narrow, P3 is clean on $0.85$ of trials, against about half for
fresh populations. Projection reaches $0.91$ there.
On the medium maze, RDM + projection, which repairs its projected plan
at every step, P3 and persistent best-of-$N$ carry their plans across
control steps and complete the detour cleanly most often ($0.24$,
$0.21$ and $0.20$, against at most $0.12$).
With $N=4$, P3 has the highest clean success on AntMaze ($0.62$, against
$0.53$ for SMC and $0.43$ for best-of-$N$).
P3 enters the disc in $29\%$ of episodes, compared with $38\%$ for SMC.
With $N=32$ the two are close ($0.61$ and $0.60$;
Appendix~\ref{app:antmaze}).

\paragraph{Planning time.}
On one RTX 4070 Ti GPU, all particles of a population are denoised in one
batch, so a planning call takes about as long as sampling one plan with
the same number of denoising iterations (Appendix~\ref{app:timing}).
A warm P3 call runs only $t_0+1$ of the $K$ iterations.
P3 takes $65$, $93$, $24$ and $186$\,ms per control step on cluttered
(without rejuvenation), narrow, the medium maze and AntMaze ($N=32$),
against $112$, $118$, $86$ and $224$\,ms for SMC.
On all four tasks it is also faster than sampling one unguided plan,
and faster than Projection, DiRecT and SafeDiffuser.
On narrow, for instance, DiRecT takes $467$\,ms and Projection $27$\,s
per step.

\subsection{Ablations}
\label{sec:ablations}

Table~\ref{tab:results} runs best-of-$N$ and SMC with and without
persistence, which separates keeping plans from selecting among them.
Keeping plans helps on its own. On cluttered, persistent best-of-$N$ is
clean on $0.87$ of trials, against $0.34$ for best-of-$N$.
Selecting among kept plans adds to this. P3 reaches $0.99$ on cluttered
and $0.85$ on narrow, against $0.87$ and $0.56$ for persistent
best-of-$N$, and the two are similar on the medium maze ($0.21$ and
$0.20$; Appendix~\ref{app:pbon}).
At the Table~\ref{tab:results} setting, persistence also lowers plan
switching, by up to $2.5$ times (Table~\ref{tab:persistence-effects}).

Without rejuvenation, P3 is clean on $0.90$ of cluttered trials, still
far above SMC, but on both-hard it falls to $0.88$, below SMC's $0.97$
(Table~\ref{tab:full-avoiding}).
The ablations in Figure~\ref{fig:ablations-row} run on cluttered with
margin $0.01$ and no rejuvenation.
With $8$ particles, P3 is cleaner than cold SMC with $32$ ($0.92$
against $0.46$) at about a seventh of the denoising cost
(Table~\ref{tab:abl-pop}).
Clean success is highest at intermediate depths ($0.97$--$0.99$ for
$t_0=4$--$12$) and falls to $0.40$ at $t_0=19$, where the re-noised
input depends only weakly on the previous plan and switching returns
to the cold-SMC level
(Theorem~\ref{thm:memory}).
At that depth, with weights carried across steps or reset to uniform,
clean success stays near cold SMC ($0.40$ and $0.43$ against $0.35$). Retaining weights alone does not reproduce
the gains at intermediate depths (Appendix~\ref{app:abl-weights}).

\begin{figure}[t]
    \centering
    \includegraphics[width=\linewidth]{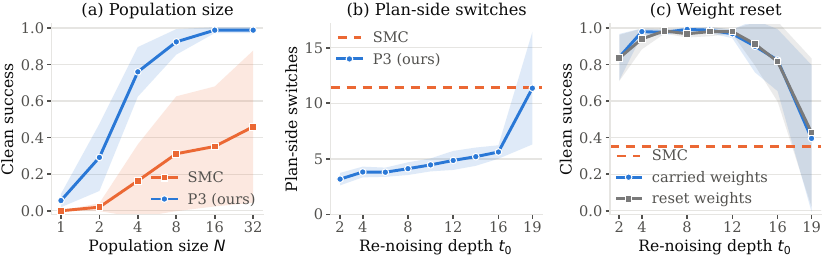}
    \vspace{-6pt}
    \caption{
        Ablations on cluttered, mean $\pm$ one standard deviation over five
        model seeds. (a) Clean success against the population size. (b)
        Plan-side switches and (c) clean success with carried or reset
        weights against the re-noising depth $t_0$, at $N=16$. Dashed: SMC.
        Details: Appendix~\ref{app:ablations}.
    }
    \label{fig:ablations-row}
    \vspace{-8pt}
\end{figure}

\vspace{-5pt}
\section{Discussion and Limitations}
\label{sec:discussion}
\label{sec:limitations}
\vspace{-3pt}

Among population methods, persistence helps most when a clean route is
rare under the model, as on cluttered and the narrow maze. A kept
population can hold such a route once found, whereas a population
regenerated at every step has to find it again.
It adds little where clean routes are common and the population is
large. On AntMaze the unguided model alone is clean on $0.38$ of trials,
and with $N=32$ mean clean success is $0.61$ for P3 and $0.60$ for SMC.
Each replan there shifts the kept plans by five steps at once
(Appendix~\ref{app:antmaze}).
When a single plan already spans the obstacles, as on the large maze
with $64$-step plans, correcting that plan by constrained optimization
also works well.

Our test-time constraints are static regions and boundaries on planar
positions.
The exact-repair and equal spherical Gaussian assumptions of
Corollary~\ref{cor:memory-routes} do not hold exactly for our learned
planner, whose repair includes shifting and learned denoising and whose
resampling creates shared ancestry (Appendix~\ref{app:kernel}).
The corollary's threshold fails to predict retention at some higher
noise levels, even at a well-separated Avoiding state, and we have not
validated it for rare routes or on the mazes.
Proposition~\ref{prop:keep}(iii) compares one-step probabilities of finding
an acceptable plan, averaged over the kept plans; it gives no pointwise,
long-run, or executed-path safety guarantee.

\vspace{-5pt}
\section{Conclusion}
\label{sec:conclusion}
\vspace{-3pt}

P3 keeps a weighted population of diffusion plans across control steps.
At every step, it shifts and partially re-noises the kept plans and
selects among them by constraint weighting and resampling, without
retraining. Under an idealized repair, we proved that the re-noising
depth controls how reliably a kept plan stays on its route. Keeping a
rare, well-separated route also takes far fewer plans than rediscovering
it. In our experiments, keeping plans helped most where clean routes are
rare, selecting among them added to this, and reuse made each replan
cheaper. A diffusion planner need not start from scratch at every step.
The plans it has already found are worth keeping.


\newpage
\section*{Reproducibility Statement}
Algorithm~\ref{alg:p3} and Appendix~\ref{app:impl} specify P3 fully,
including the potential schedule, the resampling rule, the noise-scale
convention and action selection. Appendix~\ref{app:impl} also describes
our implementations of RDM + projection, DiRecT and SafeDiffuser; the
latter two port the official code to our checkpoints.
Appendix~\ref{app:protocol} gives, for every task, the horizon, number of
denoising steps, population size, re-noising depth, margins, trial counts,
episode limits and grading rule, and Appendices~\ref{app:maze},
\ref{app:mb} and~\ref{app:antmaze} describe the task-specific models,
training data and constraint presets.
Avoiding uses the public D3IL demonstrations and AntMaze the public D4RL
\texttt{antmaze-umaze-v1} data. The PointMaze demonstrations come from the
scripted waypoint controller described in Appendices~\ref{app:maze}
and~\ref{app:mb}.
Main results use five independently trained models per task, or one
checkpoint with six blocks of evaluation seeds on AntMaze. Each trial's
start state is fixed by its index, so all methods face the same trials.
The planners' random draws differ across GPU models, so bit-exact reruns
require the same device.
Each theoretical result states its assumptions, and
Appendix~\ref{app:proofs} gives complete proofs.
Appendix~\ref{app:timing} reports the hardware and protocol of the timing
measurements.
Code, pretrained models and the scripts that produce the tables and
figures are available at the anonymized repository linked in the abstract.

\section*{AI Use Statement}
Generative AI tools were used in this work to give feedback on experimental
design, assist in analyzing and interpreting experimental results, help develop
parts of the theoretical framework, suggest some mathematical claims and key
steps in their proofs, assist in writing proofs, implement and debug software,
produce figures, search for and summarize related literature, format references,
draft sections of the manuscript, and edit the text for clarity and structure.

Other theoretical ideas and mathematical claims were developed by the authors.
We have not used generative AI tools to propose or refine research hypotheses,
translate text, or generate, clean, or reformat datasets. Qualitative or
thematic data analysis is not applicable to this work.

The authors reviewed and verified all AI-assisted technical content: mathematical
proofs were checked line-by-line by hand, software and implementations were
verified through test suites, and literature citations were checked against
original sources. The authors take full responsibility for the final content
of this work.

\bibliographystyle{iclr2027_conference}
\bibliography{references}

@book{delmoral2004feynman,
  title     = {Feynman--{K}ac Formulae: Genealogical and Interacting Particle Systems with Applications},
  author    = {Del Moral, Pierre},
  year      = {2004},
  publisher = {Springer},
  address   = {New York}
}

@book{chopin2020introduction,
  title     = {An Introduction to Sequential {M}onte {C}arlo},
  author    = {Chopin, Nicolas and Papaspiliopoulos, Omiros},
  year      = {2020},
  publisher = {Springer},
  address   = {Cham}
}

@article{delmoral2012adaptive,
  title   = {On adaptive resampling strategies for sequential {M}onte {C}arlo methods},
  author  = {Del Moral, Pierre and Doucet, Arnaud and Jasra, Ajay},
  journal = {Bernoulli},
  volume  = {18},
  number  = {1},
  pages   = {252--278},
  year    = {2012}
}

@article{gerber2019negative,
  title   = {Negative association, ordering and convergence of resampling methods},
  author  = {Gerber, Mathieu and Chopin, Nicolas and Whiteley, Nick},
  journal = {The Annals of Statistics},
  volume  = {47},
  number  = {4},
  pages   = {2236--2260},
  year    = {2019}
}

@inproceedings{wu2023practical,
  title     = {Practical and asymptotically exact conditional sampling in diffusion models},
  author    = {Wu, Luhuan and Trippe, Brian L. and Naesseth, Christian A. and Blei, David M. and Cunningham, John P.},
  booktitle = {Advances in Neural Information Processing Systems},
  volume    = {36},
  year      = {2023}
}

@inproceedings{singhal2025general,
  title     = {A general framework for inference-time scaling and steering of diffusion models},
  author    = {Singhal, Raghav and Horvitz, Zachary and Teehan, Ryan and Ren, Mengye and Yu, Zhou and McKeown, Kathleen and Ranganath, Rajesh},
  booktitle = {International Conference on Machine Learning},
  year      = {2025}
}

@inproceedings{janner2022planning,
  title     = {Planning with diffusion for flexible behavior synthesis},
  author    = {Janner, Michael and Du, Yilun and Tenenbaum, Joshua B. and Levine, Sergey},
  booktitle = {International Conference on Machine Learning},
  year      = {2022}
}

@inproceedings{chi2023diffusion,
  title     = {Diffusion policy: Visuomotor policy learning via action diffusion},
  author    = {Chi, Cheng and Feng, Siyuan and Du, Yilun and Xu, Zhenjia and Cousineau, Eric and Burchfiel, Benjamin and Song, Shuran},
  booktitle = {Robotics: Science and Systems},
  year      = {2023}
}

@inproceedings{zhou2023adaptive,
  title     = {Adaptive online replanning with diffusion models},
  author    = {Zhou, Siyuan and Du, Yilun and Zhang, Shun and Xu, Mengdi and Shen, Yikang and Xiao, Wei and Yeung, Dit-Yan and Gan, Chuang},
  booktitle = {Advances in Neural Information Processing Systems},
  volume    = {36},
  year      = {2023}
}

@inproceedings{romer2025diffusion,
  title     = {Diffusion predictive control with constraints},
  author    = {R{\"o}mer, Ralf and von Rohr, Alexander and Schoellig, Angela P.},
  booktitle = {Learning for Dynamics and Control Conference},
  year      = {2025}
}

@inproceedings{jia2024towards,
  title     = {Towards diverse behaviors: A benchmark for imitation learning with human demonstrations},
  author    = {Jia, Xiaogang and Blessing, Denis and Jiang, Xinkai and Reuss, Moritz and Donat, Atalay and Lioutikov, Rudolf and Neumann, Gerhard},
  booktitle = {International Conference on Learning Representations},
  year      = {2024}
}

@article{fu2020d4rl,
  title   = {{D4RL}: Datasets for deep data-driven reinforcement learning},
  author  = {Fu, Justin and Kumar, Aviral and Nachum, Ofir and Tucker, George and Levine, Sergey},
  journal = {arXiv preprint arXiv:2004.07219},
  year    = {2020}
}

@misc{delazcano2023gymnasium,
  title        = {Gymnasium Robotics},
  author       = {de Lazcano, Rodrigo and Andreas, Kallinteris and Tai, Jun Jet and Lee, Seungjae Ryan and Terry, Jordan},
  year         = {2023},
  howpublished = {\url{http://github.com/Farama-Foundation/Gymnasium-Robotics}}
}

@article{direct2026,
  title   = {{DiRecT}: Safe diffusion-based planning via receding-horizon denoising},
  author  = {Giaretta, Paolo and Li, Zeyang and Azizan, Navid},
  journal = {arXiv preprint arXiv:2606.15359},
  year    = {2026}
}

@article{guarniero2017iterated,
  title   = {The iterated auxiliary particle filter},
  author  = {Guarniero, Pieralberto and Johansen, Adam M. and Lee, Anthony},
  journal = {Journal of the American Statistical Association},
  volume  = {112},
  number  = {520},
  pages   = {1636--1647},
  year    = {2017}
}

@article{heng2020controlled,
  title   = {Controlled sequential {M}onte {C}arlo},
  author  = {Heng, Jeremy and Bishop, Adrian N. and Deligiannidis, George and Doucet, Arnaud},
  journal = {The Annals of Statistics},
  volume  = {48},
  number  = {5},
  pages   = {2904--2929},
  year    = {2020}
}

@article{whiteley2014twisted,
  title   = {Twisted particle filters},
  author  = {Whiteley, Nick and Lee, Anthony},
  journal = {The Annals of Statistics},
  volume  = {42},
  number  = {1},
  pages   = {115--141},
  year    = {2014}
}

@inproceedings{zhao2024probabilistic,
  title     = {Probabilistic inference in language models via twisted sequential {M}onte {C}arlo},
  author    = {Zhao, Stephen and Brekelmans, Rob and Makhzani, Alireza and Grosse, Roger},
  booktitle = {International Conference on Machine Learning},
  year      = {2024}
}

@article{wang2026trust,
  title   = {Inference-time alignment of diffusion models via trust-region iterative twisted sequential {M}onte {C}arlo},
  author  = {Wang, Weixin and Yang, Yu and Deng, Wei and Xu, Pan},
  journal = {arXiv preprint arXiv:2605.25123},
  year    = {2026}
}

@inproceedings{hoeg2025streaming,
  title     = {Streaming diffusion policy: Fast policy synthesis with variable noise diffusion models},
  author    = {H{\o}eg, Sigmund H. and Du, Yilun and Egeland, Olav},
  booktitle = {IEEE International Conference on Robotics and Automation},
  year      = {2025}
}

@inproceedings{trippe2023diffusion,
  title     = {Diffusion probabilistic modeling of protein backbones in {3D} for the motif-scaffolding problem},
  author    = {Trippe, Brian L. and Yim, Jason and Tischer, Doug and Baker, David and Broderick, Tamara and Barzilay, Regina and Jaakkola, Tommi},
  booktitle = {International Conference on Learning Representations},
  year      = {2023}
}

@inproceedings{cardoso2024monte,
  title     = {Monte {C}arlo guided denoising diffusion models for {B}ayesian linear inverse problems},
  author    = {Cardoso, Gabriel and Idrissi, Yazid Janati El and Le Corff, Sylvain and Moulines, Eric},
  booktitle = {International Conference on Learning Representations},
  year      = {2024}
}

@inproceedings{dou2024diffusion,
  title     = {Diffusion posterior sampling for linear inverse problem solving: A filtering perspective},
  author    = {Dou, Zehao and Song, Yang},
  booktitle = {International Conference on Learning Representations},
  year      = {2024}
}

@inproceedings{kim2025test,
  title     = {Test-time alignment of diffusion models without reward over-optimization},
  author    = {Kim, Sunwoo and Kim, Minkyu and Park, Dongmin},
  booktitle = {International Conference on Learning Representations},
  year      = {2025}
}

@article{li2024derivative,
  title   = {Derivative-free guidance in continuous and discrete diffusion models with soft value-based decoding},
  author  = {Li, Xiner and Zhao, Yulai and Wang, Chenyu and Scalia, Gabriele and Eraslan, Gokcen and Nair, Surag and Biancalani, Tommaso and Ji, Shuiwang and Regev, Aviv and Levine, Sergey and Uehara, Masatoshi},
  journal = {arXiv preprint arXiv:2408.08252},
  year    = {2024}
}

@inproceedings{skreta2025feynman,
  title     = {Feynman--{K}ac correctors in diffusion: Annealing, guidance, and product of experts},
  author    = {Skreta, Marta and Akhound-Sadegh, Tara and Ohanesian, Viktor and Bondesan, Roberto and Aspuru-Guzik, Al{\'a}n and Doucet, Arnaud and Brekelmans, Rob and Tong, Alexander and Neklyudov, Kirill},
  booktitle = {International Conference on Machine Learning},
  year      = {2025}
}

@article{ma2025inference,
  title   = {Inference-time scaling for diffusion models beyond scaling denoising steps},
  author  = {Ma, Nanye and Tong, Shangyuan and Jia, Haolin and Hu, Hexiang and Su, Yu-Chuan and Zhang, Mingda and Yang, Xuan and Li, Yandong and Jaakkola, Tommi and Jia, Xuhui and Xie, Saining},
  journal = {arXiv preprint arXiv:2501.09732},
  year    = {2025}
}

@inproceedings{corso2024particle,
  title     = {Particle guidance: Non-{I.I.D.} diverse sampling with diffusion models},
  author    = {Corso, Gabriele and Xu, Yilun and De Bortoli, Valentin and Barzilay, Regina and Jaakkola, Tommi},
  booktitle = {International Conference on Learning Representations},
  year      = {2024}
}

@inproceedings{ajay2023conditional,
  title     = {Is conditional generative modeling all you need for decision-making?},
  author    = {Ajay, Anurag and Du, Yilun and Gupta, Abhi and Tenenbaum, Joshua B. and Jaakkola, Tommi and Agrawal, Pulkit},
  booktitle = {International Conference on Learning Representations},
  year      = {2023}
}

@inproceedings{carvalho2023motion,
  title     = {Motion planning diffusion: Learning and planning of robot motions with diffusion models},
  author    = {Carvalho, Jo{\~a}o and Le, An T. and Baierl, Mark and Koert, Dorothea and Peters, Jan},
  booktitle = {IEEE/RSJ International Conference on Intelligent Robots and Systems},
  year      = {2023}
}

@inproceedings{pan2024model,
  title     = {Model-based diffusion for trajectory optimization},
  author    = {Pan, Chaoyi and Yi, Zeji and Shi, Guanya and Qu, Guannan},
  booktitle = {Advances in Neural Information Processing Systems},
  volume    = {37},
  year      = {2024}
}

@inproceedings{xiao2025safediffuser,
  title     = {{SafeDiffuser}: Safe planning with diffusion probabilistic models},
  author    = {Xiao, Wei and Wang, Tsun-Hsuan and Gan, Chuang and Hasani, Ramin and Lechner, Mathias and Rus, Daniela},
  booktitle = {International Conference on Learning Representations},
  year      = {2025}
}

@inproceedings{christopher2024constrained,
  title     = {Constrained synthesis with projected diffusion models},
  author    = {Christopher, Jacob K. and Baek, Stephen and Fioretto, Ferdinando},
  booktitle = {Advances in Neural Information Processing Systems},
  volume    = {37},
  year      = {2024}
}

@inproceedings{lou2023reflected,
  title     = {Reflected diffusion models},
  author    = {Lou, Aaron and Ermon, Stefano},
  booktitle = {International Conference on Machine Learning},
  year      = {2023}
}

@article{fishman2023diffusion,
  title   = {Diffusion models for constrained domains},
  author  = {Fishman, Nic and Klarner, Leo and De Bortoli, Valentin and Mathieu, Emile and Hutchinson, Michael},
  journal = {Transactions on Machine Learning Research},
  year    = {2023}
}

@inproceedings{dhariwal2021diffusion,
  title     = {Diffusion models beat {GANs} on image synthesis},
  author    = {Dhariwal, Prafulla and Nichol, Alexander},
  booktitle = {Advances in Neural Information Processing Systems},
  volume    = {34},
  year      = {2021}
}

@inproceedings{williams2017information,
  title     = {Information theoretic {MPC} for model-based reinforcement learning},
  author    = {Williams, Grady and Wagener, Nolan and Goldfain, Brian and Drews, Paul and Rehg, James M. and Boots, Byron and Theodorou, Evangelos A.},
  booktitle = {IEEE International Conference on Robotics and Automation},
  year      = {2017}
}

@inproceedings{lambert2020stein,
  title     = {Stein variational model predictive control},
  author    = {Lambert, Alexander and Fishman, Adam and Fox, Dieter and Boots, Byron and Ramos, Fabio},
  booktitle = {Conference on Robot Learning},
  year      = {2020}
}

@inproceedings{chua2018deep,
  title     = {Deep reinforcement learning in a handful of trials using probabilistic dynamics models},
  author    = {Chua, Kurtland and Calandra, Roberto and McAllister, Rowan and Levine, Sergey},
  booktitle = {Advances in Neural Information Processing Systems},
  volume    = {31},
  year      = {2018}
}

@incollection{kantas2009sequential,
  title     = {Sequential {M}onte {C}arlo for model predictive control},
  author    = {Kantas, Nikolas and Maciejowski, Jan M. and Lecchini-Visintini, Andrea},
  booktitle = {Nonlinear Model Predictive Control},
  series    = {Lecture Notes in Control and Information Sciences},
  volume    = {384},
  publisher = {Springer},
  year      = {2009}
}

@inproceedings{shafiullah2022behavior,
  title     = {Behavior transformers: Cloning $k$ modes with one stone},
  author    = {Shafiullah, Nur Muhammad and Cui, Zichen and Altanzaya, Ariuntuya and Pinto, Lerrel},
  booktitle = {Advances in Neural Information Processing Systems},
  volume    = {35},
  year      = {2022}
}

@article{mayne2000constrained,
  title   = {Constrained model predictive control: Stability and optimality},
  author  = {Mayne, David Q. and Rawlings, James B. and Rao, Christopher V. and Scokaert, Pierre O. M.},
  journal = {Automatica},
  volume  = {36},
  number  = {6},
  pages   = {789--814},
  year    = {2000}
}

@article{gordon1993novel,
  title   = {Novel approach to nonlinear/non-{G}aussian {B}ayesian state estimation},
  author  = {Gordon, Neil J. and Salmond, David J. and Smith, Adrian F. M.},
  journal = {IEE Proceedings F (Radar and Signal Processing)},
  volume  = {140},
  number  = {2},
  pages   = {107--113},
  year    = {1993}
}

@inproceedings{zhao2023learning,
  title     = {Learning fine-grained bimanual manipulation with low-cost hardware},
  author    = {Zhao, Tony Z. and Kumar, Vikash and Levine, Sergey and Finn, Chelsea},
  booktitle = {Robotics: Science and Systems},
  year      = {2023}
}

@inproceedings{black2025realtime,
  title     = {Real-time execution of action chunking flow policies},
  author    = {Black, Kevin and Galliker, Manuel and Levine, Sergey},
  booktitle = {Advances in Neural Information Processing Systems},
  year      = {2025}
}

@inproceedings{liu2025bidirectional,
  title     = {Bidirectional decoding: Improving action chunking via guided test-time sampling},
  author    = {Liu, Yuejiang and Hamid, Jubayer Ibn and Xie, Annie and Lee, Yoonho and Du, Max and Finn, Chelsea},
  booktitle = {International Conference on Learning Representations},
  year      = {2025}
}

@article{liu1994covariance,
  title   = {Covariance structure of the {G}ibbs sampler with applications to the comparisons of estimators and augmentation schemes},
  author  = {Liu, Jun S. and Wong, Wing Hung and Kong, Augustine},
  journal = {Biometrika},
  volume  = {81},
  number  = {1},
  pages   = {27--40},
  year    = {1994}
}

@inproceedings{meng2022sdedit,
  title     = {{SDEdit}: Guided Image Synthesis and Editing with Stochastic Differential Equations},
  author    = {Meng, Chenlin and He, Yutong and Song, Yang and Song, Jiaming and Wu, Jiajun and Zhu, Jun-Yan and Ermon, Stefano},
  booktitle = {International Conference on Learning Representations},
  year      = {2022}
}

@article{gilks2001following,
  title   = {Following a moving target---{M}onte {C}arlo inference for dynamic {B}ayesian models},
  author  = {Gilks, Walter R. and Berzuini, Carlo},
  journal = {Journal of the Royal Statistical Society: Series B (Statistical Methodology)},
  volume  = {63},
  number  = {1},
  pages   = {127--146},
  year    = {2001}
}

@book{thrun2005probabilistic,
  title     = {Probabilistic Robotics},
  author    = {Thrun, Sebastian and Burgard, Wolfram and Fox, Dieter},
  publisher = {MIT Press},
  year      = {2005}
}

@article{blackwell1953equivalent,
  title   = {Equivalent comparisons of experiments},
  author  = {Blackwell, David},
  journal = {The Annals of Mathematical Statistics},
  volume  = {24},
  number  = {2},
  pages   = {265--272},
  year    = {1953}
}

@inproceedings{yang2024diffusion,
  title     = {Diffusion-{ES}: Gradient-free Planning with Diffusion for Autonomous and Instruction-guided Driving},
  author    = {Yang, Brian and Su, Huangyuan and Gkanatsios, Nikolaos and Ke, Tsung-Wei and Jain, Ayush and Schneider, Jeff and Fragkiadaki, Katerina},
  booktitle = {Proceedings of the IEEE/CVF Conference on Computer Vision and Pattern Recognition (CVPR)},
  year      = {2024}
}

@inproceedings{piche2019probabilistic,
  title     = {Probabilistic Planning with Sequential {M}onte {C}arlo Methods},
  author    = {Pich{\'e}, Alexandre and Thomas, Valentin and Ibrahim, Cyril and Bengio, Yoshua and Pal, Chris},
  booktitle = {International Conference on Learning Representations},
  year      = {2019}
}

@article{mizuta2025unified,
  title   = {Unified Generation-Refinement Planning: Bridging Guided Flow Matching and Sampling-Based {MPC} for Social Navigation},
  author  = {Mizuta, Kazuki and Leung, Karen},
  journal = {arXiv preprint arXiv:2508.01192},
  year    = {2025}
}

@inproceedings{jung2025joint,
  title     = {Joint Model-based Model-free Diffusion for Planning with Constraints},
  author    = {Jung, Wonsuhk and Mishra, Utkarsh A. and Ranawaka Arachchige, Nadun and Chen, Yongxin and Xu, Danfei and Kousik, Shreyas},
  booktitle = {Conference on Robot Learning},
  year      = {2025}
}

\appendix
\section{Proofs and implementation diagnostics}
\label{app:proofs}
\subsection{Setting and notation}
\label{app:setting}

Plans lie in $\mathbb{R}^{H\times T}$ with the Euclidean norm.
$\TV(\mu,\nu)=\sup_A|\mu(A)-\nu(A)|$ is the total-variation distance
between probability measures, and $\Phi$ is the standard normal cdf.
$C_k$ pins the observed coordinates at control step $k$ (and goal
coordinates on the mazes), and
$\hat x_t(y)=C_k(\operatorname{clip}(D_\theta(y,t,s_k)))$, where
$\operatorname{clip}$ clips PointMaze predictions to $[-1,1]$ and is
the identity on Avoiding and AntMaze.
Write $R_{k,t}$ for the pinned learned reverse kernel at $t\ge1$,
$F_k(y)=\hat x_0(y)$ for the terminal map, and
$B_k(z,\cdot)$ for shift, re-noising and pinning in
eq.~\eqref{eq:renoise}.
The unweighted warm repair is
$\Mimpl_k=B_kR_{k,t_0}\cdots R_{k,1}F_k$; a fresh plan has law
$p_k=\mu^{(k)}R_{k,K-1}\cdots R_{k,1}F_k$, where
$\mu^{(k)}=\Law(C_k(\sigma_{\mathrm{init}}\epsilon))$.
These kernels are specified by Algorithm~\ref{alg:p3}.

Let $Y_0\sim\mu^{(0)}$ and let $Y_r$ be the noisy-plan state at
event $r$, as in eq.~\eqref{eq:fk}.
Within a call, event $r$ at level $t(r)$ has transition
$Q_r=R_{k,t(r)+1}$; at a boundary it has
$Q_r(y,\cdot)=B_k(F_{k-1}(y),\cdot)$.
Thus $\Pref=\mu^{(0)}\prod_{r\ge1}Q_r$ is the reference chain in
eq.~\eqref{eq:fk}.
We distinguish the potential score $c^{\mathrm{pot}}$ from the
selector score $c^{\mathrm{sel}}$ when their margins differ, and set
$\psi_t(y)=e^{-\Lambda_t c^{\mathrm{pot}}(\hat x_t(y))}$.
Finite predictions give $0<\psi_t\le1$.
Plain P3 means P3 without rejuvenation.

For (E), fix $p=p_k$ and write
\begin{equation}
    \snr_t=\sigma_{\mathrm{init}}^{-1}
        \sqrt{\frac{\bar\alpha_t}{1-\bar\alpha_t}},
    \qquad
    G_t(x,\cdot)=\mathcal{N}(\snr_t x,I),
    \qquad
    \Mideal(x,dx')=\int G_{t_0}(x,dy)\,p(dx'\mid y).
    \label{eq:snr}
\end{equation}
The cosine schedule makes $\snr_t$ strictly decreasing;
at $K=20$ and $\sigma_{\mathrm{init}}=0.5$,
$\snr_0=22.28$, $\snr_{10}=1.689$, and $\snr_{19}=0.0049$.
The history $\mathcal G_k$ contains the observations, earlier
populations and random draws, and any replaced slots before call $k$.
Assumption (S) conditions on this history and removes within-call
weighting and resampling.
The Feynman--Kac identification fixes the observations externally;
conditioning on an entire endogenous closed-loop observation sequence
does not give the same law.

\subsection{The particle target}
\label{app:proof-target}

\begin{proposition}[Feynman--Kac representation]
\label{prop:target}
For fixed observations, plain P3 with carried weights is the SMC
particle system with ESS-triggered systematic resampling for
eq.~\eqref{eq:fk}.
(a) At $N=1$ it executes the unweighted warm sampler, whatever the
constraints; at $N=2$ constraints act only through the final choice
between two unweighted lineages. Both statements hold pathwise in
closed loop.
(b) If a boundary kernel forgets its input, the target from that
boundary onward is the restarted target. Given the pre-boundary
history, carried weights leave the first-event unnormalized
importance sum unbiased and multiply its variance by $N/\mathrm{ESS}$
relative to uniform weights.
\end{proposition}

\begin{proof}
At event $r$, particles move by $Q_r$, multiply their entering weights
by $\psi_{t(r)}$, and resample if $\mathrm{ESS}<N/2$.
No proposal ratio is needed. Systematic resampling is unbiased:
the interval of particle $i$, of length $W_i$, contains $NW_i$
resampling points in expectation. These are the SMC operations for
$\Gamma_R=\mu^{(0)}\psi_{t(0)}\prod_{r=1}^R Q_r\psi_{t(r)}$
\citep{delmoral2004feynman,chopin2020introduction}.
Execution takes the minimum-score member, not a draw from $\Pi_R$.

(a) For normalized nonnegative weights, $\mathrm{ESS}\ge1$, so the
strict trigger cannot fire at $N\le2$.
Moves ignore the scores and execution ignores the weights.
Couple runs with the same Gaussian draws and environment noise:
at $N=1$ induction over calls gives identical actions under every
constraint configuration; at $N=2$ it gives the planner that keeps two
unweighted slots and selects their minimum-score member.
The same argument applies to RDM as published, whose triggers never
read the constraints (Appendix~\ref{app:rdm}).

(b) Let event $b$ have $Q_b(y,\cdot)=\eta_k$.
The path measure factorizes into the pre-boundary factor and
$\eta_k(dY_b)\psi_{t_0}(Y_b)\prod_{r>b}Q_r\psi_{t(r)}$.
Normalization cancels the first factor, giving the restarted target.
Given $\mathcal G_k$, carried weights $W_i$ are fixed and
$Y_b^{(i)}\sim\eta_k$ are i.i.d. Hence, for bounded $f$ and
$\psi=\psi_{t_0}$,
\begin{equation}
    \widehat S(f)=\sum_i W_i\psi(Y_b^{(i)})f(Y_b^{(i)}),\qquad
    \mathbb E[\widehat S(f)\mid\mathcal G_k]=\eta_k(\psi f),\qquad
    \operatorname{Var}(\widehat S(f)\mid\mathcal G_k)
    =\frac{\operatorname{Var}_{\eta_k}(\psi f)}{\mathrm{ESS}(W)}.
    \label{eq:target-var}
\end{equation}
Uniform weights replace $\mathrm{ESS}(W)$ by $N$.
The normalized ratio $\widehat S(f)/\widehat S(1)$ is generally
biased; neither the unbiasedness nor the exact variance factor is
claimed for it or for later events.
\end{proof}

\paragraph{Current-plan marginal.}
\label{rem:marginal}
Let $R_k$ end call $k$, $x^k=F_k(Y_{R_k})$, and $\Pref_k=\Law_\Pref(x^k)$.
The current-plan marginal is
$\pi_k(dx)\propto\Psi_k(x)\Pref_k(dx)$, with
$\Psi_k(x)=\mathbb E_\Pref[\prod_{r\le R_k}\psi_{t(r)}(Y_r)\mid x^k=x]$.
Thus P3 reweights its reference law, but need not preserve its
conditional law on feasible plans: their intermediate predictions
can have different violations.
Terminal-only weighting gives
$\pi_k(dx)\propto e^{-\Lambda_0c^{\mathrm{pot}}(x)}\Pref_k(dx)$;
telescoped potentials give that target as well
\citep{wu2023practical,singhal2025general}.

\paragraph{Weight resetting and forgetting.}
\label{rem:target-reset}
Resetting discards the accumulated lineage weights: up to common
normalization, each is the product of potentials since its most
recent resampling, or initialization if it has never been resampled.
Resetting weights can help or hurt.
The implemented boundary never forgets exactly, but for two plans
$z,z'$ the Gaussian total-variation formula and contraction under
Markov kernels give
\[
    \TV(B_k(z,\cdot),B_k(z',\cdot))
    \le 2\Phi\!\left(\tfrac12\snr_{t_0}\|Sz-Sz'\|\right)-1 .
\]
Thus dependence on any fixed pair of shifted plans vanishes as
$\snr_{t_0}\to0$. This is a one-boundary statement, not an episode bound.

\paragraph{Scope.}
For fixed observations and bounded violations, standard consistency
results apply with multinomial resampling, including adaptive
resampling under the conditions of \citet{delmoral2012adaptive}.
Stratified resampling at every event is also consistent, whereas
systematic resampling can fail depending on particle order
\citep{gerber2019negative}; we claim no convergence theorem for our
slot order or closed-loop feedback.
Rejuvenation also lies outside Proposition~\ref{prop:target}:
newcomers replace weight-selected slots and receive assigned weights.

\subsection{Proofs of the retention results}
\label{app:proof-memory}

Under (E), with $p_Y$ the noisy marginal, the joint law is
\begin{equation}
    p(dx)G_{t_0}(x,dy)p(dx'\mid y)
    =p_Y(dy)p(dx\mid y)p(dx'\mid y).
    \label{eq:setting-exch}
\end{equation}
Thus $(x,x')$ is exchangeable and both marginals are $p$.
For any $f\in L^2(p)$, conditional independence gives the
data-augmentation identity \citep{liu1994covariance}
\begin{equation}
    \operatorname{Cov}(f(x),f(x'))
    =\operatorname{Var}(\mathbb E[f(x)\mid y]).
    \label{eq:memory-cov}
\end{equation}

\begin{proof}[Proof of Theorem~\ref{thm:memory}]
(i) Put $g(y)=p(A\mid y)$.
Since $\mathbb E g=\nu$,
$\Pr(x,x'\in A)=\mathbb E[g^2]=\nu^2+V_A$.
Dividing by $\nu$ gives retention; exchangeability gives crossing
probability $2(\nu-\nu^2-V_A)$.
The Gaussian likelihood is positive, so $0<g<1$ and
$V_A<\nu(1-\nu)$.

(ii) Write $\gamma=\snr_{t_0}$ and consider $0<\gamma'<\gamma$.
Couple $z=x+\gamma^{-1}\epsilon$ and $z'=z+c\epsilon'$,
where $c^2=\gamma'^{-2}-\gamma^{-2}$ and the noises are independent.
The noisier channel is a garbling of the first
\citep{blackwell1953equivalent}.
For $U=\mathbb E[f(x)\mid z]$, total variance gives
\[
    \operatorname{Var}(U)
    =\operatorname{Var}(\mathbb E[f(x)\mid z'])
      +\mathbb E[\operatorname{Var}(U\mid z')] .
\]
If equality holds between the first two terms, then
$U=g_1(z)=g_2(z+c\epsilon')$ almost surely.
The conditional Gaussian has an everywhere-positive density, so
Fubini's theorem forces $g_2$, and hence $U$, to be constant.
It follows that the convolution of the finite signed measure
$(f-\mathbb Ef)\,dp$ with $\mathcal N(0,\gamma^{-2}I)$ vanishes.
Its Fourier transform is multiplied by
$e^{-\|\omega\|^2/(2\gamma^2)}\ne0$, so uniqueness of Fourier
transforms implies $f=\mathbb Ef$ $p$-almost surely.
For $f=\mathbf1_A$ this contradicts $\nu\in(0,1)$.
Therefore $V_A$ increases strictly with $\gamma$, and decreases
strictly with $t_0$.

For the fresh limit, total variation contracts under the posterior
kernel $y\mapsto p(\cdot\mid y)$.
Writing $\Phi$ for the standard normal cdf, convexity and the
equal-covariance Gaussian formula give
\begin{equation}
    \TV(\Mideal(x,\cdot),p)
    \le \TV(G_{t_0}(x,\cdot),p_Y)
    \le \int\left[2\Phi\!\left(
        \frac{\gamma\|x-\tilde x\|}{2}\right)-1\right]p(d\tilde x).
    \label{eq:memory-fresh}
\end{equation}
The integrand is bounded by one and tends to zero with $\gamma$,
so dominated convergence gives convergence to $p$ for every $x$.
A second application, using
$V_A=\int_A(\Mideal(x,A)-\nu)p(dx)$, gives $V_A\to0$.
\end{proof}

\begin{proof}[Proof of Corollary~\ref{cor:memory-routes}]
Let $L\in\{a,b\}$ be the mixture label and draw $(L',x')$ jointly
from the posterior given $y$.
This realizes (E), with $L'$ the repaired plan's component.
Given $L$, the rescaled observation has law
$\mathcal N(\gamma\mu_L,\bar\sigma^2I)$,
$\bar\sigma^2=1+\gamma^2s^2$.
Bayes' rule gives
\[
    p(L=a\mid y)=
    \operatorname{sig}\!\left(\operatorname{logit}\pi+
      \frac{\|y-\gamma\mu_b\|^2-\|y-\gamma\mu_a\|^2}
           {2\bar\sigma^2}\right),
    \qquad \operatorname{sig}(u)=\frac1{1+e^{-u}} .
\]
Under $L=a$, the likelihood-ratio term is
$\kappa^2/2+\kappa\xi$ with $\xi\sim\mathcal N(0,1)$.
Consequently
\begin{equation}
    \Pr(L'=a\mid L=a)
    =\mathbb E\!\left[
       \operatorname{sig}(\operatorname{logit}\pi+
                          \kappa^2/2+\kappa\xi)\right].
    \label{eq:memory-keep}
\end{equation}
The function $u\mapsto\mathbb E[\operatorname{sig}(u+\kappa\xi)]$
is strictly increasing and equals $1/2$ at $u=0$ by symmetry.
Retention is therefore at least $1/2$ iff
$\operatorname{logit}\pi+\kappa^2/2\ge0$, proving the threshold.
For independent fresh draws, discovery probability is
$1-(1-\pi)^N$, giving the stated count.
At $\pi=0.01$ the thresholds are $3.0315$ and $68.97$,
hence $69$ draws.
For $s>0$, $\kappa<\|\mu_a-\mu_b\|/s$, so reducing the noise
cannot overcome insufficient intrinsic route separation.
\end{proof}

\paragraph{Mixture components and route regions.}
Let $D=\mu_a-\mu_b\ne0$ and
$A=\{x:\langle x-(\mu_a+\mu_b)/2,D\rangle>0\}$.
For $s>0$, the probability that $\mathbf1_A(x)$ disagrees with
$\mathbf1_{\{L=a\}}$ is $e=\Phi(-\|D\|/(2s))$.
By stationarity the same is true after repair, so
$|\Pr(x,x'\in A)-\Pr(L=L'=a)|\le2e$.
Thus a geometric interpretation requires small overlap; for
conditional retention of a rare route, the error must also be small
relative to its mass.

\paragraph{Linear-feature diagnostic.}
\label{cor:memory-linear}
For a unit vector $u$ with $s_u^2=\operatorname{Var}(u^\top x)\in(0,\infty)$,
eq.~\eqref{eq:memory-cov} implies
\[
    \operatorname{corr}(u^\top x,u^\top x')
    =1-\frac{\mathbb E[(u^\top x-\mathbb E[u^\top x\mid y])^2]}{s_u^2}
    \ge \frac{s_u^2\gamma^2}{1+s_u^2\gamma^2}.
\]
Indeed, the conditional mean minimizes squared error; restricting
to affine predictors from $u^\top y=\gamma u^\top x+u^\top\epsilon$
gives error $s_u^2/(1+s_u^2\gamma^2)$.
Figure~\ref{fig:kernel}(a) compares the measured correlations with this
lower bound.
It does not bound switching for non-Gaussian plans; the indicator
identity of Theorem~\ref{thm:memory} does.

\subsection{Kept and fresh plans}
\label{app:proof-keep}

Here $c_k=c^{\mathrm{sel}}$ and $A_k=\{c_k\le\delta\}$.
On the mazes $\delta=0$ is admissible.
On Avoiding the score includes consistency equalities that do not
vanish exactly; $0<\delta\le m_{\mathrm v}$, the virtual padding,
ensures acceptance implies satisfaction of the unpadded geometric
rows. An acceptable plan need not produce a clean executed step
because of tracking error, and a clean episode need not have an
acceptable full-horizon plan at every step.

\begin{proof}[Proof of Proposition~\ref{prop:keep}]
(i) Given $\mathcal G_k$, the independent failure events have
probabilities $1-r_k(z_i)$ and $1-\nu_k$; multiplication gives
eq.~\eqref{eq:keep-formula}.
The minimum-score plan lies in $A_k$ iff any member does.
The fresh-plan recovery clause follows by considering only the $n$
fresh draws; their mutual independence suffices, regardless of the
held plans. If every $r_k(z_i)\ge r$, the failure probability is at
most $(1-r)^m(1-\nu_k)^n$.

(ii) The union bound gives
$\Pr(\bigcup_{i=1}^N\{X_i\in A_k\}\mid\mathcal G_k)
\le\sum_i\Pr(X_i\in A_k\mid\mathcal G_k)=N\nu_k$.
Under independence with a common rate it equals $1-(1-\nu_k)^N$.
The bounds condition only on the history before this call and
therefore hold in closed loop.
For a success probability of at least $1/2$ the independent sampler
needs $N\ge\ln2/\ln(1/(1-\nu_k))$.
For cold SMC, $\nu_k$ in the union bound is its mean final-particle
rate, which can depend on $N$ and differs from the unweighted
model's rate.

(iii) At a fixed observation, integrate over independent held plans
$z_i\sim p(\cdot\mid A)$ and independent repair noise.
Theorem~\ref{thm:memory}(i) gives each repair acceptance
$a=\nu+V_A/\nu\ge\nu$.
Together with independent fresh draws this gives
$(1-a)^m(1-\nu)^n\le(1-\nu)^{m+n}$.
\end{proof}

Best-of-$N$ satisfies (S) exactly with $(m,n)=(0,N)$.
Plain P3 uses $(N,0)$ and rejuvenation $(N-n,n)$ only under (S).
In the implementation, resampling both changes the marginal laws
and creates shared ancestors.
Newcomers are generated by a weighted cold sub-chain and enter the
combined population before further selection; they are not $n$
independent final draws.
Thus the recovery clause motivates rejuvenation but is not a bound
for Table~\ref{tab:results}'s implementation.
The one-step comparison also requires the specific distribution
$p(\cdot\mid A)$, not an arbitrary selection of acceptable plans.

\subsection{Implementation diagnostics}
\label{app:kernel}

\paragraph{Departures from exact repair.}
\label{rem:memory-inexact}
Even an exact posterior-mean network would not implement (E):
a learned DDPM step draws one Gaussian around its predicted mean,
whereas the exact reverse transition mixes these Gaussians over
$p(x\mid y_t)$; the terminal step returns a prediction rather than a
posterior draw.
The network is also trained at unit noise scale but sampled at
$\sigma_{\mathrm{init}}=0.5$.
Shifting and the changing observation add further departures in
closed loop. The shallow-depth failure in
Table~\ref{tab:abl-t0} is consistent with retaining an infeasible
continuation, but does not identify the cause or locate an optimal
depth from the theory alone.

\paragraph{Repair at fixed states.}
We use $17$ Avoiding states from all five model seeds, including ten
two-sided decision states chosen before the measurement by a pilot
side fraction in $[0.2,0.8]$.
At each state we draw $4096$ cold plans, re-noise each without shifting,
and denoise under the same observation for every $t_0=0,\ldots,19$.
Side is the sign of the final lateral position relative to the
centerline. We record switching relative to $2\pi(1-\pi)$, where
$\pi$ is the cold side fraction, and the repaired side fraction $\pi'$.
A second draw uses independent noise.
For correlation, the first half of the cold sample fits two principal
directions and the second half evaluates them; the maze diagnostic
uses three states per maze and one direction.
Routes are designated well separated when the distance between
side means divided by their within-side spread is at least six,
a threshold fixed before the run.

\begin{figure}[htbp]
    \centering
    \includegraphics[width=\linewidth]{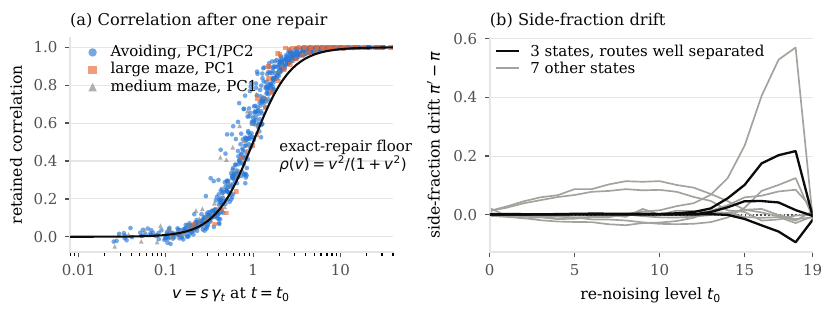}
    \caption{
        Learned repair at fixed states.
        (a) Correlation along held-out plan directions against the
        exact-repair floor derived in Appendix~\ref{app:proof-memory}.
        (b) Side-fraction drift at the ten decision states; (E) predicts
        zero. Black: the three states whose two routes are well separated.
    }
    \label{fig:kernel}
\end{figure}

Figure~\ref{fig:theory}(b) shows the main draw's switch ratios.
At the three separated-route states, ratios stay at or below
$0.03$ through $t_0=12$; at $t_0=19$ all ten ratios are in
$[0.98,1.03]$.
Figure~\ref{fig:kernel}(b) shows that learned repair is not stationary:
side-fraction drift reaches
$0.115$ at $t_0\le12$ and $0.57$ at $t_0=18$.
The switch rate reaches $1.39$ times the fresh-draw value.
Measured correlations are consistent with the exact-repair lower
bound at moderate depths, but some fall below it at larger depths.
The side-fraction changes show that learned repair does not preserve
the plan distribution, as (E) would require.

We also estimate $\kappa$ in Corollary~\ref{cor:memory-routes} from
the side means and pooled within-side spread along their difference.
At $t_0=10$, all ten decision states pass the estimated threshold
and have minority-side retention above $1/2$ in both draws.
The threshold is not reliable at larger depths. At one well-separated
state, $t_0=17$ gives $\kappa=1.77$ against a threshold of $1.57$,
but measured retention is $0.38$ and $0.41$ in the two draws.
These moment-based estimates do not establish a two-Gaussian fit;
the tested minority masses are about $0.20$--$0.46$, not $0.01$.

\paragraph{Acceptable plans at visited states.}
\label{app:find-keep}
We log cold SMC and P3 at $N=16$, without rejuvenation:
on cluttered, $10$ trials each for two model seeds and one contrast seed,
$t_0=10$ and $\delta=0.01$; on the narrow large maze,
five trials each for two model seeds, $H=64$, $t_0=16$ and $\delta=0$.
Avoiding reuses the ablation trials; the maze diagnostics use separate trials.
At every visited state, $64$ extra unweighted cold plans estimate
$\nu_k$ independently of the planner's samples and are never executed.
Four unweighted repairs of each previously acceptable held plan
estimate bare repair retention.
We also record the acceptable fraction among its final descendants
after weighting and resampling (zero if none survive).
In Table~\ref{tab:find-keep}, $\hat q_k^{\mathrm{desc}}$ denotes this
descendant fraction and $\hat r_k^{\mathrm{rep}}$ the bare-repair
fraction; neither is the exact kernel probability in
Proposition~\ref{prop:keep}.

\begin{table}[htbp]
    \centering
    \small
    \caption{
        Find/keep diagnostics. Lost all: cold-SMC episodes with no
        acceptable plan at some step; $\nu_k$ there: median estimated
        fresh rate at such lapses. Retention columns report medians.
        Held: fraction of P3 steps with an acceptable plan among states
        with estimated $\nu_k<1/N$; fresh: mean plug-in estimate
        $1-(1-\nu_k)^N$ at those same states.
    }
    \label{tab:find-keep}
\begin{tabular}{@{}llcccccccc@{}}
\toprule
 & & \multicolumn{2}{c}{clean} & \multicolumn{2}{c}{cold SMC} & \multicolumn{2}{c}{P3, all steps} & \multicolumn{2}{c}{P3, $\nu_k<1/N$} \\
\cmidrule(lr){3-4}\cmidrule(lr){5-6}\cmidrule(lr){7-8}\cmidrule(l){9-10}
Task & Seed & cold & P3 & lost all & $\nu_k$ there & $\hat q_k^{\mathrm{desc}}$ & $\hat r^{\mathrm{rep}}_k$ & held & fresh \\
\midrule
Cluttered & 1 & 0/10 & 10/10 & 10/10 & 0.016 & 1.00 & 0.95 & 0.82 & 0.13 \\
Cluttered & 2 & 0/10 & 9/10 & 10/10 & 0.016 & 1.00 & 0.94 & 0.61 & 0.22 \\
Cluttered & 3 & 8/10 & 9/10 & 7/10 & 0.031 & 1.00 & 0.96 & 0.54 & 0.17 \\
Narrow maze & 1 & 0/5 & 2/5 & 5/5 & 0.000 & 1.00 & 1.00 & 0.27 & 0.05 \\
Narrow maze & 2 & 0/5 & 4/5 & 5/5 & 0.000 & 1.00 & 1.00 & 0.30 & 0.07 \\
\bottomrule
\end{tabular}

\end{table}

On the first two cluttered seeds, cold SMC loses every acceptable plan
in all $20$ logged episodes, at lapse states with median estimated
fresh rate $0.016$.
At P3's own states with estimated $\nu_k<1/N$, it holds an acceptable
plan on $0.82$ and $0.61$ of steps, against fresh estimates $0.13$
and $0.22$.
The table's retention columns summarize all steps. On the separate
subset of hard states with estimated $\nu_k<1/N$, median bare-repair
retention is only $0.05$--$0.40$ across the five settings, below the
corresponding post-resampling descendant fractions.
Cold SMC likewise exceeds the independent-fresh estimate at its own
hard states. Assumption (S) excludes this weighting and resampling.
The fresh column is an estimate, not a bound for cold SMC, and a zero
rate means no successes among $64$ probes, not zero model mass.

\subsection{Projection and reweighting}
\label{app:proof-transport}

Let $F$ be a closed, nonempty feasible set. An exact nearest-point
projection of $x\notin F$ lies in $\partial F$:
an interior nearest point could be moved slightly toward $x$ while
remaining feasible.
Consequently, if $p(\partial F)=0$, the projected law places at least
$p(F^c)$ mass on a $p$-null boundary.
A selection target $G(x)p(dx)/p(G)$, for measurable $G\ge0$ with
$0<p(G)<\infty$, is instead absolutely continuous with respect to $p$.
It equals $p(\cdot\mid F)$ for $G=\mathbf1_F$ when $p(F)>0$,
but an arbitrary path-dependent weight need not preserve feasible
route shares (Appendix~\ref{app:proof-target}).
Zero clearance here is relative to the projected set; padding may
give positive clearance from the original constraints, and equality
rows can make all of $F$ its boundary.

This distributional distinction does not predict the switching rate
of the local solver used inside the learned denoising process.

\section{Implementation details}
\label{app:impl}

Algorithm~\ref{alg:p3} gives P3 for joint state--action plans with one
executed action per call. PointMaze also pins the goal coordinates and
clips clean predictions to $[-1,1]$; Avoiding does not clip them.
AntMaze uses unclipped state-only plans, obtains actions from inverse
dynamics, and executes and shifts by five steps per call
(Appendix~\ref{app:antmaze}).
Cold SMC runs the cold branch at every control step, best-of-$N$ skips
lines~\ref{alg:weight}--\ref{alg:resample-end}, and the weight-reset
ablation adds the optional reset of line~\ref{alg:reset}; rejuvenation
(Section~\ref{sec:rejuvenation}) is the optional step of
line~\ref{alg:rejuv}.

\begin{algorithm}[t]
\caption{Persistent particle planning (P3). Every assignment applies to
all particles $i=1,\ldots,N$.}
\label{alg:p3}
\begin{algorithmic}[1]
\Require model $D_\theta$ with $K$ reverse steps, residual $c$, strengths
    $\Lambda_t=\beta_t\lambda_t$, population size $N$, re-noising level
    $t_0$, noise scale $\sigma_{\mathrm{init}}$; $\epsilon^{(i)},\xi^{(i)}$
    denote fresh $\mathcal{N}(0,I)$ draws
\State $\mathrm{warm}\gets\textsc{false}$
\For{control step $k=0,1,2,\ldots$}
    \State observe $s_k$
    \If{$\mathrm{warm}=\textsc{false}$}
        \Comment{cold start from noise}
        \State $y^{(i)}\gets\sigma_{\mathrm{init}}\,\epsilon^{(i)}$;
            $\ \log w^{(i)}\gets0$; $\ T\gets K-1$
    \Else
        \Comment{persistence: shift, re-noise, keep the weights}
        \State $y^{(i)}\gets\sqrt{\bar\alpha_{t_0}}\,S(x^{(i)})
            +\sigma_{\mathrm{init}}\sqrt{1-\bar\alpha_{t_0}}\,\epsilon^{(i)}$;
            $\ T\gets t_0$
            \Comment{eqs.~\eqref{eq:shift} and~\eqref{eq:renoise}}
        \State \textbf{optional} (weight-reset ablation): $\log w^{(i)}\gets0$
            \label{alg:reset}
        \State \textbf{optional} (rejuvenation): replace the $y^{(i)}$ of the
            $n$ lowest-weight particles by fresh plans, run by this procedure
            with $n$ particles from $K-1$ to $t_0+1$, and set their
            $\log w^{(i)}$ to the log of the population's mean weight
            \label{alg:rejuv}
    \EndIf
    \State pin the observed coordinates of $y^{(i)}$ to $s_k$
    \For{$t=T,T-1,\ldots,0$}
        \Comment{denoising clock}
        \State $\hat{x}^{(i)}\gets D_\theta(y^{(i)},t,s_k)$, pinned to $s_k$
        \State $\log w^{(i)}\gets\log w^{(i)}-\Lambda_t\,c(\hat{x}^{(i)})$
            \label{alg:weight}
            \Comment{eqs.~\eqref{eq:potential} and~\eqref{eq:weight-update}}
        \State $W^{(i)}\gets\exp(\log w^{(i)})\big/\sum_j\exp(\log w^{(j)})$
        \If{$1\big/\sum_i(W^{(i)})^2<N/2$}
            \Comment{eq.~\eqref{eq:ess}}
            \State draw ancestors $a^{(i)}$ by systematic resampling from $W$
            \State $(y^{(i)},\hat{x}^{(i)})\gets(y^{(a^{(i)})},\hat{x}^{(a^{(i)})})$;
                $\ \log w^{(i)}\gets0$
        \Else
            \State $\log w^{(i)}\gets\log W^{(i)}$
                \label{alg:resample-end}
        \EndIf
        \If{$t>0$}
            \State $y^{(i)}\gets m_t(\hat{x}^{(i)},y^{(i)})
                +\sigma_{\mathrm{init}}e^{\ell_t/2}\,\xi^{(i)}$, pinned to $s_k$
                \Comment{reverse step}
        \Else
            \State $x^{(i)}\gets\hat{x}^{(i)}$
                \Comment{final plans}
        \EndIf
    \EndFor
    \State $i^\star\gets\arg\min_i c(x^{(i)})$, lowest index on ties
        \Comment{selection}
    \State execute the first action of $x^{(i^\star)}$;
        keep $\{x^{(i)},\log w^{(i)}\}_{i=1}^{N}$;
        $\ \mathrm{warm}\gets\textsc{true}$
\EndFor
\end{algorithmic}
\end{algorithm}

\paragraph{Constraint residuals.}
In the Avoiding experiments the residuals $c_j$ of
eq.~\eqref{eq:plan-cost}, aggregated by their maximum, encode geometric exclusions (the physical
pillars and the added test-time regions and boundaries), action bounds,
and planning-model consistency between consecutive desired positions
and the action, $x^{\mathrm{des}}_{h+1}=x^{\mathrm{des}}_h+a_h$.
The geometric residuals are those of the violation metric of
Section~\ref{sec:baselines}, evaluated on the measured-position
channel with the margins of Appendix~\ref{app:protocol}, after the
affine map from normalized to physical coordinates.

\paragraph{Potential schedule.}
The coefficients of the potential in eq.~\eqref{eq:potential} follow
\begin{equation}
    u_t
    =
    u_{\mathrm{start}}
    +
    (u_{\mathrm{end}}-u_{\mathrm{start}})
    \left(1-\frac{t}{K-1}\right)^{1.5},
    \qquad u\in\{\lambda,\beta\},
    \label{eq:schedule}
\end{equation}
with
$(\lambda_{\mathrm{start}},\lambda_{\mathrm{end}})=(0.5,15)$ and
$(\beta_{\mathrm{start}},\beta_{\mathrm{end}})=(0.1,15)$, so that
$\Lambda_t=\beta_t\lambda_t$ rises from $0.05$ at $t=K-1$ to $225$ at
$t=0$.
A warm call resumes this schedule at $t_0$ rather than compressing it
into fewer iterations, so the potential applied at a given level is the
same in cold and warm calls.

\paragraph{Weights and resampling.}
The effective sample size
\begin{equation}
    \operatorname{ESS}_t
    =
    \frac{1}{\sum_i (W_t^{(i)})^2}
    \label{eq:ess}
\end{equation}
is computed after every update, and when $\operatorname{ESS}_t<N/2$
ancestors are drawn by systematic resampling from $W_t$.
Resampling copies the selected ancestors and resets weights to uniform;
otherwise the normalized weights are retained.
P3 carries the population and its weights into the next call, continuing
the weight product of Proposition~\ref{prop:target} across the boundary.
The weight-reset ablation instead starts each call with uniform weights.

\paragraph{Indexing of warm calls.}
Levels are zero-based: a cold call runs $t=K-1,\ldots,0$ and a warm
call runs $t=t_0,\ldots,0$, so a warm call performs $t_0+1$ denoiser
evaluations, and $t_0=10$ with $K=20$ corresponds to 11 iterations.
Re-noising uses the cumulative signal coefficient $\bar\alpha_{t_0}$
of the cosine schedule at index $t_0$.

\paragraph{Noise-scale convention.}
All checkpoints are sampled with half the DDPM noise scale:
the chain is initialized at $\sigma_{\mathrm{init}}\,\epsilon$ with
$\epsilon\sim\mathcal{N}(0,I)$, and the reverse noise injected at level
$t$ has standard deviation $\sigma_{\mathrm{init}}e^{\ell_t/2}$, where
$\ell_t$ is the DDPM posterior log-variance and
$\sigma_{\mathrm{init}}=0.5$.
Equation~\eqref{eq:renoise} uses the same sampling noise scale,
$\sigma_{\mathrm{init}}=0.5$, rather than unit noise scale.
Using $\sigma=1$ at the same $t_0$ would reduce the squared
signal-to-noise ratio $\snr_{t_0}^2$ of eq.~\eqref{eq:snr} by a factor
of four.
All compared methods use this sampling noise scale.

\paragraph{Action selection.}
The executed plan is the population member with minimal $c$ after the
final denoising step, with ties resolved by the lowest particle index.

\paragraph{RDM baseline.}
\label{app:rdm}
RDM \citep{zhou2023adaptive} executes one plan.
After every environment step it writes the observed state into the plan
and scores the partially executed plan by the KL terms of the model's
variational bound at small noise levels.
It regenerates the plan from noise every $H$ steps or when the
normalized score falls below $l_s$, repairs it by re-noising and
denoising under the current state (dropping the executed transitions)
when the score falls below $l_f>l_s$, and otherwise keeps executing it,
with actions from inverse dynamics towards the plan's next state.
Every run uses the unguided row's sampler, model seeds, trials and step
caps.
RDM leaves three details open: we normalize the score as
$\ell=\min(1,L_{\mathrm{gen}}/L)$, so that a fresh plan scores one;
we use the two smallest nondegenerate KL levels of our chains; and we
write states only into the plan, because the action channels of the
Avoiding checkpoints are not consistent with their planned setpoints.
As published, RDM does not see the added constraints: its paths do not
depend on them, because they never change the dynamics its likelihood
monitors.

\paragraph{RDM with constraints.}
We add constraint projection during repair and regeneration, and
trigger replanning when a retained plan violates an added constraint.
RDM's step ``constrain the plan with the conditions'', used to pin the
current state, also applies Projection's projector and margins.
A kept plan is checked against the constraints padded by the selection
margin, following RDM's rule that an infeasible plan must be replanned.
A violation triggers a regeneration.
On cluttered, the projector's SLSQP failed on $71\%$ of repairs and
$15\%$ of regenerations in the same episodes.
We use the thresholds $(0.5,0.98)$ with repair depth $9$ on Avoiding,
repair at every step with depth $5$ on the medium maze, and
$(0.9,0.99)$ with depth $25$ on AntMaze.
On the large maze we use RDM's published thresholds and plan $10$
trials per seed. The narrow result includes $50$ completed trials.

\paragraph{DiRecT baseline.}
\label{app:direct}
We port the official DiRecT code \citep{direct2026} to our checkpoints
and run it with the unguided row's sampler, model seeds, trials and step
caps, replanning at every control step (every five on AntMaze).
In the last half of the denoising steps ($10$ of $20$; $25$ of $50$ on
AntMaze) it takes the reverse step, solves
$x^\star=\arg\min_{x\in S}\tfrac12\lambda\bar\alpha_{t-1}\beta_t^{-1}
\|x-\hat x_0\|^2+C(x)$ with IPOPT ($\lambda=1$), and adds
$\sqrt{\bar\alpha_{t-1}}(x^\star-\hat x_0)$ to the noisy plan; the last
step returns $x^\star$.
As in the official code, $\hat x_0$ is the clean estimate at the current
plan on Avoiding and at the uncorrected step on the mazes, and $S$
includes linear dynamics fitted by least squares to the model's training
transitions (none on AntMaze, whose plans have no actions).
On Avoiding, $S$ holds the added constraints and the pillars padded by
DiRecT's $0.025$\,m and $C=0$; on cluttered this padding closes the
corridor, so we use Projection's $0.01$.
On the mazes, $S$ holds the obstacles with DiRecT's margin ($0.4$ added
to the superellipse power sum) and actions within their data range, and
$C$ is the squared path length; on AntMaze, $S$ is the planning set of
the other rows.
Constraints apply to every planned state but the current one.
IPOPT did not converge on under $0.3\%$ of solves ($5.0\%$ on AntMaze),
where the last iterate is used, as in the official code.

\paragraph{SafeDiffuser baseline.}
\label{app:safediffuser}
We port the official SafeDiffuser code \citep{xiao2025safediffuser} to our
checkpoints and run it with the unguided row's sampler, model seeds,
trials and step caps.
At every denoising step each planned position moves by the step closest
to the unguided update that satisfies $\nabla b\cdot u+b\ge0$, with $b$
the constraint's power sum minus $1.01$ (the official robust term) or,
for a boundary, its signed distance; the other channels follow the
unguided update, and the QPs are solved with qpth as in the official code.
Of its variants, RoS imposes this at every step, ReS relaxes it (penalty
weight $100$) except in the last steps, and TVS grows the barrier over
the last steps; we map the official $256$-step schedules to our chains by
their fraction, and use RoS on both-hard and the medium maze and ReS
elsewhere.
As in its official evaluation, the executed action tracks the plan's next
waypoint, $(p_1-p)+(v_1-v)$ on PointMaze; on Avoiding the tracked
waypoints are the commanded setpoints, and on AntMaze the inverse-dynamics head acts on the
plan.
Avoiding and AntMaze, where SafeDiffuser defines no setting, use the
other rows' padded constraints, on AntMaze only the added disc (the
official code has no wall barriers); the mazes use the obstacles with the
robust term only.
A waypoint QP that is infeasible, or that asks for a move wider than the
normalized workspace (a vanishing barrier gradient inside a large
power-4 obstacle), keeps the unguided update; this affected at most
$4.9\%$ of waypoint QPs (right-hard) and under $0.3\%$ on the mazes.
Relaxation and fallback updates do not guarantee barrier satisfaction.
The position corrections also leave velocities and actions unchanged,
so a tracked trajectory can violate constraints even when its planned
positions do not.

\section{Experimental protocol}
\label{app:protocol}

\paragraph{Models, horizons and re-noising.}
For each task and model seed, all methods share the same pretrained
model. Cold calls use $K=20$ reverse iterations, or $K=50$ on AntMaze.
The standard execution interval is one action, or five on AntMaze;
RDM instead checks its likelihood after each action and replans
adaptively (Appendix~\ref{app:rdm}).
The planning horizon is $H=8$ on Avoiding and the medium maze, $H=32$ on
AntMaze and $H=64$ on the large maze.
Population methods use $N=16$ particles, $N=32$ in the Avoiding columns
and $N=4$ in the AntMaze column of Table~\ref{tab:results} ($N=32$ on
AntMaze: Appendix~\ref{app:antmaze}).
P3 re-noises to $t_0=10$ on Avoiding, $16$ on the large maze, $4$ on the
medium maze and $40$ of $50$ on AntMaze; in Table~\ref{tab:results} it rejuvenates
$n=8$ particles per step on Avoiding only.

\paragraph{Constraints and margins.}
Avoiding aggregates its residuals by their maximum and the mazes sum
the obstacle penetration of the plan's positions.
On Avoiding, potentials and selectors pad the added constraints by
$0.01$, or $0.04$ for best-of-$N$, SMC and P3 in
Table~\ref{tab:avoiding-performance}, and the pillars by $0.025$;
Projection pads its constraints by $0.025$ ($0.01$ for the cluttered
regions), with approximate dynamics rows that couple its corrections to
the actions.
On cluttered, two radius-$0.08$ discs have centers $0.20$ apart;
padding each by $0.025$ closes the passage between them, so Projection
uses $0.01$ instead.
On the mazes the sampling methods pad the obstacles by $0.01$ on the
large layout and by $0.3$ on the medium layout. Projection pads them by the ball radius $0.1$; on
AntMaze every constraint-aware planner keeps a $1.0$ clearance from the
disc.
On Avoiding and PointMaze, grading uses the unpadded constraints on the
measured state before each planning call; the terminal post-action
state is not included. AntMaze grades the recorded center path after
every executed action, including the terminal state.

\paragraph{Trials.}
Avoiding and PointMaze use five model seeds and $50$ trials per seed,
method and configuration ($100$ for the Avoiding population
rows of Table~\ref{tab:avoiding-performance}). Computational cost limits
Projection on the large maze to $5$--$10$ trials per seed.
AntMaze uses six evaluation seeds of $20$ trials, paired across methods
on their environment seeds.
Episodes last at most $200$ control steps on Avoiding, $1200$ on the
large maze and $800$ on the medium maze.
Methods are paired by trial seed, which fixes each episode's start state.

\paragraph{Timing.}
The Avoiding ablation latencies were measured separately from the
success-rate runs, with one process on an idle CPU, one model seed
and ten trials per setting. Population-size and noise-depth sweeps
use separate timing runs. Appendix~\ref{app:timing} compares all
methods on one GPU.

\section{Full results}
\label{app:full-results}

Tables~\ref{tab:task-success} and~\ref{tab:full-avoiding} give task
success and accumulated violation for the runs in Table~\ref{tab:results}.
Tables~\ref{tab:maze-performance} and~\ref{tab:persistence-effects}
report plan switches and episode lengths.

\begin{table}[t]
    \centering
    \footnotesize
    \caption{Task success for the runs in Table~\ref{tab:results},
        irrespective of constraint violations. Means and standard
        deviations are across model seeds (evaluation seeds on AntMaze);
        trial counts are as in Table~\ref{tab:full-avoiding}.}
    \label{tab:task-success}
    \resizebox{\linewidth}{!}{
\setlength{\tabcolsep}{4pt}
\begin{tabular}{@{}lcccccccc@{}}
\toprule
 & \multicolumn{4}{c}{Avoiding} & \multicolumn{2}{c}{Large maze} & Medium & AntMaze \\
\cmidrule(lr){2-5}\cmidrule(lr){6-7}\cmidrule(lr){8-8}\cmidrule(l){9-9}
Method & Both & Cluttered & Left & Right & Broad & Narrow & Blocked & Turn \\
\midrule
\multicolumn{9}{@{}l}{\textit{Task success $\uparrow$}} \\
\addlinespace[2pt]
Unguided & \result{0.80}{0.11} & \result{0.79}{0.11} & \result{0.80}{0.11} & \result{0.79}{0.11} & \result{\mathbf{1.00}}{0.00} & \result{\mathbf{1.00}}{0.00} & \result{\mathbf{0.76}}{0.17} & \result{0.82}{0.08} \\
RDM + projection & \result{0.77}{0.10} & \result{0.85}{0.05} & \result{0.79}{0.10} & \result{0.63}{0.10} & \result{0.88}{0.12} & \result{0.98}{0.04} & \result{0.58}{0.10} & \result{0.84}{0.06} \\
Projection & \result{0.87}{0.05} & \result{0.97}{0.02} & \result{0.95}{0.02} & \result{0.88}{0.08} & \result{\mathbf{1.00}}{0.00} & \result{\mathbf{1.00}}{0.00} & \result{0.26}{0.13} & \result{0.83}{0.08} \\
DiRecT & \result{\mathbf{1.00}}{0.00} & \result{0.72}{0.12} & \result{0.88}{0.03} & \result{0.93}{0.07} & \result{\mathbf{1.00}}{0.00} & \result{\mathbf{1.00}}{0.00} & \result{0.12}{0.11} & \result{0.85}{0.09} \\
SafeDiffuser & \result{0.50}{0.21} & \result{0.53}{0.29} & \result{0.66}{0.32} & \result{0.44}{0.36} & \result{0.41}{0.43} & \result{0.52}{0.36} & \result{0.53}{0.10} & \result{0.78}{0.06} \\
Best-of-$N$ & \result{0.99}{0.01} & \result{0.65}{0.25} & \result{\mathbf{1.00}}{0.00} & \result{0.95}{0.05} & \result{\mathbf{1.00}}{0.00} & \result{\mathbf{1.00}}{0.00} & \result{0.55}{0.11} & \result{0.85}{0.05} \\
SMC & \result{0.99}{0.00} & \result{0.82}{0.16} & \result{\mathbf{1.00}}{0.00} & \result{0.97}{0.04} & \result{\mathbf{1.00}}{0.00} & \result{\mathbf{1.00}}{0.00} & \result{\mathbf{0.76}}{0.13} & \result{0.86}{0.13} \\
Persistent best-of-$N$ & \result{0.98}{0.03} & \result{0.96}{0.04} & \result{\mathbf{1.00}}{0.00} & \result{0.92}{0.13} & \result{\mathbf{1.00}}{0.00} & \result{\mathbf{1.00}}{0.00} & \result{0.35}{0.11} & \result{0.83}{0.09} \\
P3 (ours) & \result{\mathbf{1.00}}{0.00} & \result{\mathbf{1.00}}{0.00} & \result{\mathbf{1.00}}{0.00} & \result{\mathbf{0.98}}{0.04} & \result{\mathbf{1.00}}{0.00} & \result{\mathbf{1.00}}{0.00} & \result{0.45}{0.09} & \result{\mathbf{0.89}}{0.07} \\
\bottomrule
\end{tabular}
}
\end{table}

\begin{table}[t]
    \centering
    \footnotesize
    \caption{
        Accumulated violation, mean $\pm$ standard deviation across five
        model seeds (clean success is in Table~\ref{tab:results}).
        Avoiding: $50$ trials per seed ($100$ for the population rows);
        best-of-$N$, SMC and P3 use $N=32$ and margin $0.04$, P3 with eight
        fresh particles per step (Section~\ref{sec:baselines}); the last
        Avoiding block is P3 at the same setting without rejuvenation.
        Mazes: large maze ($H=64$), medium maze ($H=8$) and AntMaze U
        ($H=32$), $50$ trials per seed (AntMaze: six evaluation seeds of
        $20$ paired trials; large-maze Projection: $5$--$10$); population
        methods use $N=16$, except on AntMaze ($N=4$).
        Unguided, RDM + projection and Projection use one sample;
        persistent best-of-$N$ uses P3's setting without weighting or
        resampling (Appendix~\ref{app:pbon}); RDM + projection uses $10$ trials per seed
        on the large maze (Appendix~\ref{app:rdm}).
    }
    \label{tab:full-avoiding}
    \label{tab:full-maze}
\setlength{\tabcolsep}{4pt}
\begin{tabular}{@{}lcccc@{}}
\toprule
 & \multicolumn{4}{c}{Avoiding} \\
\cmidrule(l){2-5}
Method & Both-hard & Cluttered & Left-hard & Right-hard \\
\midrule
\multicolumn{5}{@{}l}{\textit{Accumulated violation $\downarrow$}} \\
\addlinespace[2pt]
Unguided & \result{2.717}{0.788} & \result{2.463}{0.582} & \result{1.484}{2.506} & \result{6.992}{3.522} \\
RDM + projection & \result{0.295}{0.086} & \result{0.594}{0.274} & \result{0.646}{0.361} & \result{1.030}{0.987} \\
Projection & \result{0.035}{0.040} & \result{0.016}{0.015} & \result{0.025}{0.019} & \result{0.215}{0.188} \\
DiRecT & \result{\mathbf{0.000}}{0.000} & \result{0.021}{0.014} & \result{0.004}{0.001} & \result{0.003}{0.004} \\
SafeDiffuser & \result{0.009}{0.004} & \result{0.848}{1.637} & \result{\mathbf{0.000}}{0.000} & \result{4.633}{5.890} \\
Best-of-$N$ & \result{0.001}{0.001} & \result{0.014}{0.013} & \result{0.004}{0.002} & \result{0.014}{0.027} \\
SMC & \result{\mathbf{0.000}}{0.001} & \result{0.013}{0.017} & \result{0.004}{0.004} & \result{\mathbf{0.000}}{0.000} \\
Persistent best-of-$N$ & \result{0.001}{0.001} & \result{0.002}{0.002} & \result{0.006}{0.005} & \result{\mathbf{0.000}}{0.000} \\
P3 (ours) & \result{\mathbf{0.000}}{0.000} & \result{\mathbf{0.000}}{0.000} & \result{0.001}{0.001} & \result{\mathbf{0.000}}{0.000} \\
\midrule
\multicolumn{5}{@{}l}{\textit{P3 without rejuvenation}} \\
\addlinespace[2pt]
Clean success $\uparrow$ & \result{0.88}{0.12} & \result{0.90}{0.10} & \result{0.94}{0.05} & \result{1.00}{0.00} \\
Accumulated violation $\downarrow$ & \result{0.090}{0.118} & \result{0.055}{0.053} & \result{0.002}{0.001} & \result{0.000}{0.000} \\
\bottomrule
\end{tabular}

    \vspace{6pt}
\setlength{\tabcolsep}{4pt}
\begin{tabular}{@{}lcccc@{}}
\toprule
 & \multicolumn{2}{c}{Large maze} & Medium maze & AntMaze U \\
\cmidrule(lr){2-3} \cmidrule(lr){4-4} \cmidrule(lr){5-5}
Method & Broad & Narrow & Blocked & Turn \\
\midrule
\multicolumn{5}{@{}l}{\textit{Accumulated violation $\downarrow$}} \\
\addlinespace[2pt]
Unguided & \result{5.684}{2.430} & \result{14.442}{5.039} & \result{18.396}{0.786} & \result{0.188}{0.097} \\
RDM + projection & \result{2.319}{2.273} & \result{1.530}{0.754} & \result{6.505}{1.336} & \result{0.177}{0.062} \\
Projection & \result{3.977}{3.019} & \result{1.447}{1.979} & \result{4.005}{1.992} & \result{0.148}{0.052} \\
DiRecT & \result{\mathbf{0.000}}{0.000} & \result{\mathbf{0.001}}{0.002} & \result{\mathbf{0.000}}{0.000} & \result{0.145}{0.035} \\
SafeDiffuser & \result{0.404}{0.468} & \result{12.062}{7.333} & \result{15.544}{1.378} & \result{0.174}{0.071} \\
Best-of-$N$ & \result{0.389}{0.605} & \result{3.589}{3.935} & \result{14.999}{1.515} & \result{0.131}{0.026} \\
SMC & \result{0.528}{1.021} & \result{3.506}{4.444} & \result{18.272}{1.060} & \result{0.126}{0.024} \\
Persistent best-of-$N$ & \result{0.169}{0.109} & \result{2.525}{2.994} & \result{4.229}{3.323} & \result{0.173}{0.042} \\
P3 (ours) & \result{0.031}{0.021} & \result{0.246}{0.454} & \result{7.092}{3.988} & \result{\mathbf{0.091}}{0.029} \\
\bottomrule
\end{tabular}

\end{table}

\begin{table}[t]
    \centering
    \small
    \caption{
        Large maze: plan switches (changes of the selected plan's end
        cell) per episode and episode
        length for the runs of Table~\ref{tab:full-maze}, mean $\pm$
        standard deviation across five model seeds.
    }
    \label{tab:maze-performance}
\begin{tabular}{@{}lcccc@{}}
\toprule
 & \multicolumn{2}{c}{Plan switches $\downarrow$} & \multicolumn{2}{c}{Episode length (steps) $\downarrow$} \\
\cmidrule(lr){2-3} \cmidrule(lr){4-5}
Method & Broad & Narrow & Broad & Narrow \\
\midrule
Unguided & \result{72.6}{15.9} & \result{72.5}{15.9} & \result{\mathbf{529}}{39} & \result{528}{38} \\
Projection & \result{58.6}{17.2} & \result{37.8}{10.5} & \result{535}{36} & \result{432}{60} \\
Best-of-$N$ & \result{68.3}{18.6} & \result{44.9}{11.0} & \result{586}{39} & \result{426}{5} \\
SMC & \result{69.1}{21.8} & \result{43.2}{13.8} & \result{615}{50} & \result{421}{16} \\
P3 (ours) & \result{\mathbf{44.5}}{11.7} & \result{\mathbf{28.2}}{4.8} & \result{620}{84} & \result{\mathbf{408}}{6} \\
\bottomrule
\end{tabular}

\end{table}

\paragraph{Avoiding.}
Table~\ref{tab:persistence-effects} compares plan-side switching under
Projection, SMC and P3. At the Table~\ref{tab:results} setting,
Projection's SLSQP solver fails on
$1$--$9\%$ of calls. Figure~\ref{fig:route} shows executed paths on
cluttered and on the medium maze.

\paragraph{Margin sensitivity.}
Table~\ref{tab:margin-sensitivity} varies each planner's Avoiding margin
without changing its other settings or the unpadded grading.
The base margin is $0.025$ on both-hard, left-hard and right-hard, and
$0.01$ on cluttered. P3 at the base margin retains its advantage over
Projection on cluttered and right-hard. On both-hard, increasing
Projection's margin to $0.04$ improves its clean success from $0.834$
to $0.942$, leaving a smaller gap to P3's $0.986$.
The same increase harms Projection on left-hard and right-hard, where
its SLSQP fallback rate rises from $0.033$ to $0.233$ and from $0.040$
to $0.339$, respectively. Projection is not tested with the larger
margin on cluttered because the padded discs block its narrow passage
(Appendix~\ref{app:protocol}). P3 can use $0.04$ there because its
margin enters a soft potential and selector, not a hard projection.

\begin{table}[t]
    \centering
    \footnotesize
    \caption{Avoiding margin sensitivity: clean success, mean $\pm$
        standard deviation across five model seeds and $100$ trials per
        seed. Base margin is $0.025$, except $0.01$ on cluttered. P3 uses
        $N=32$ and eight fresh particles per step throughout.}
    \label{tab:margin-sensitivity}
    \setlength{\tabcolsep}{4pt}
    \begin{tabular}{@{}lcccc@{}}
        \toprule
        & \multicolumn{2}{c}{Projection} & \multicolumn{2}{c}{P3} \\
        \cmidrule(lr){2-3} \cmidrule(lr){4-5}
        Case & Base & $0.04$ & Base & $0.04$ \\
        \midrule
        Both-hard & \result{0.834}{0.063} & \result{0.942}{0.044}
                  & \result{0.944}{0.036} & \result{0.986}{0.015} \\
        Cluttered & \result{0.446}{0.270} & n/a
                  & \result{0.980}{0.009} & \result{0.994}{0.008} \\
        Left-hard & \result{0.824}{0.112} & \result{0.364}{0.240}
                  & \result{0.892}{0.103} & \result{0.936}{0.037} \\
        Right-hard & \result{0.796}{0.095} & \result{0.108}{0.135}
                  & \result{0.966}{0.043} & \result{0.978}{0.044} \\
        \bottomrule
    \end{tabular}
\end{table}

\begin{figure}[t]
    \centering
    \includegraphics[width=\linewidth]{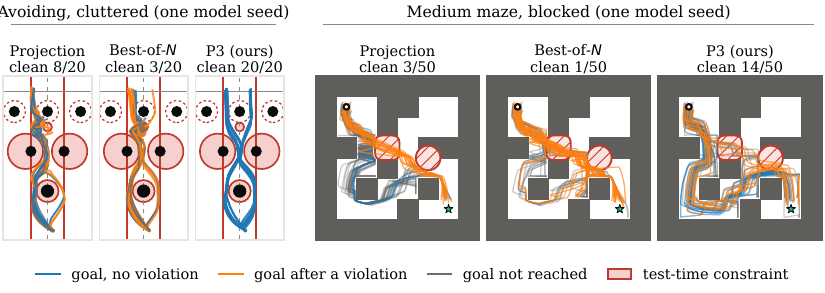}
    \caption{
        Executed paths on cluttered (twenty trials per method on the model
        seed whose Projection clean success is closest to its five-seed mean)
        and on the medium maze (the 50 trials per method of one model seed,
        within $0.07$ of the five-seed mean clean success for every method).
        Cluttered uses $N=16$, $t_0=10$, potential margin $0.01$ and no
        rejuvenation, rather than the Table~\ref{tab:results} setting;
        the medium-maze runs use that table's setting.
        The plots show side changes around the exclusion regions on
        cluttered (left) and trajectories through or around the blocked
        corridor on the medium maze (right).
    }
    \label{fig:route}
\end{figure}

\begin{table}[t]
    \centering
    \small
    \caption{
        Projection, SMC and P3 on Avoiding at their Table~\ref{tab:results}
        settings. SMC and P3 use $N=32$, margin $0.04$ and $100$ trials
        per seed; Projection uses one plan and $50$ trials per seed.
        Entries are means $\pm$ standard deviations across five model seeds.
    }
    \label{tab:persistence-effects}
    \begin{tabular}{@{}lccc@{}}
        \toprule
        Constraint configuration & Projection & SMC & P3 (ours) \\
        \midrule
        \multicolumn{4}{@{}l}{
            \textit{Plan-side switches per episode $\downarrow$}
        } \\
        \addlinespace[2pt]
        Both-hard
        & \result{5.34}{2.70}
        & \result{4.49}{1.56}
        & \result{\mathbf{1.79}}{0.90} \\
        Cluttered
        & \result{12.42}{2.56}
        & \result{10.63}{2.92}
        & \result{\mathbf{7.72}}{0.91} \\
        Left-hard
        & \result{7.28}{3.36}
        & \result{2.87}{1.02}
        & \result{\mathbf{2.35}}{0.62} \\
        Right-hard
        & \result{11.10}{4.09}
        & \result{3.31}{0.95}
        & \result{\mathbf{2.63}}{0.46} \\
        \bottomrule
    \end{tabular}
\end{table}

\begin{figure}[t]
    \centering
    \includegraphics[width=\linewidth]{
        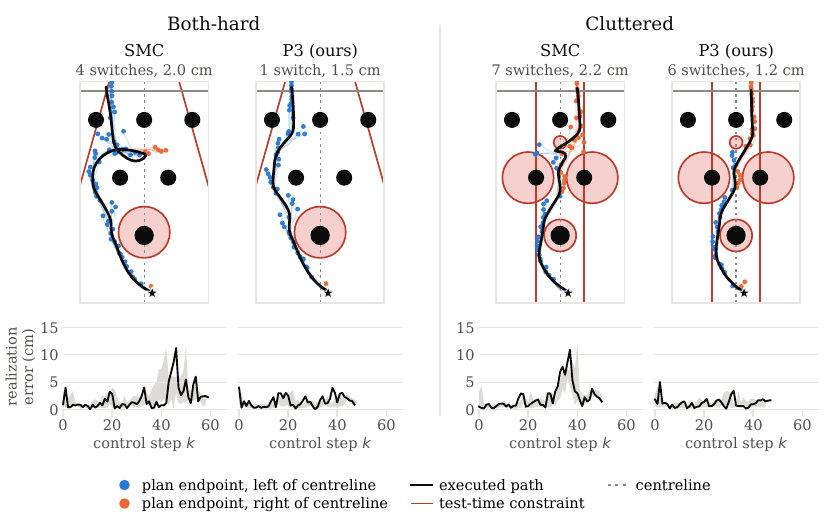
    }
    \caption{
        One example episode of SMC and P3 on both-hard and cluttered.
        These runs use $N=16$, $t_0=10$, potential margin $0.01$ and
        no rejuvenation, with eight episodes per method and configuration.
        Top: the executed path, with each selected plan's endpoint coloured by
        its side of the centerline, so a colour change is a plan-side switch.
        Bottom: realization error, the distance between each plan's endpoint and the position reached
        $H-1$ steps later (black), with the 10--90\% band of the other
        seven episodes of the same method and configuration.
    }
    \label{fig:switches-rollout}
\end{figure}

\paragraph{Persistence at matched settings.}
Table~\ref{tab:persistence-effects} compares P3 with SMC at the same
$N$ and margin.
P3 switches plan side less often in every configuration, by up to
$2.5$ times. Accumulated violations are reported in
Table~\ref{tab:full-avoiding}.
Figure~\ref{fig:switches-rollout} shows single episodes: on both-hard
SMC keeps changing its planned continuation near the upper obstacles,
and the plotted P3 episode on cluttered has smaller realization errors.
Realization error is the distance between a planned endpoint and the position
reached $H-1$ steps later.

\paragraph{Large maze.}
DiRecT has the lowest mean penetration on both presets
(Table~\ref{tab:full-maze}). Among the population methods, P3 has the
lowest mean penetration, including $0.031$ on broad against $0.389$
for best-of-$N$, and the fewest endpoint-cell switches.
On narrow, its mean penetration is less than one tenth that of the
fresh populations. Appendix~\ref{app:maze} gives further large-maze
details.

\paragraph{AntMaze.}
At $N=4$, P3 enters the disc less often than SMC and Projection ($29\%$
of episodes, against $38\%$ and $48\%$), has the lowest mean
penetration and reaches the goal as often as SMC ($0.89$ and $0.86$;
Tables~\ref{tab:task-success} and~\ref{tab:full-maze}).

\subsection{Planning time}
\label{app:timing}
Table~\ref{tab:timing} times every method of Table~\ref{tab:results} on
one configuration per task. All cells ran one after another on an
otherwise idle machine with one RTX 4070 Ti GPU and an i7-13700F CPU.
The diffusion model runs on the GPU and denoises all particles of a
population in one batch. The solvers of Projection, DiRecT and
SafeDiffuser run on the CPU.
Each cell runs three episodes of $50$ control steps with one model seed
(two episodes of ten steps for Projection and RDM + projection on
narrow), and the first step of every episode, which includes the cold
start, is not timed.
On AntMaze, a control step is one replan of five environment steps.
RDM + projection keeps its plan on most steps and repairs or regenerates
it on the others, so its step time varies widely. On narrow it kept its
plan at every timed step.

\begin{table}[t]
    \centering
    \footnotesize
    \caption{Wall-clock time per control step in milliseconds, median and
        mean over the timed steps, on one RTX 4070 Ti. Each method uses its
        Table~\ref{tab:results} setting, except that the AntMaze population
        rows are timed at $N=32$. P3 rejuvenates only on Avoiding,
        and the last row times it there without rejuvenation.
    }
    \label{tab:timing}
\setlength{\tabcolsep}{4pt}
\begin{tabular}{@{}lrrrrrrrr@{}}
\toprule
 & \multicolumn{2}{c}{Cluttered} & \multicolumn{2}{c}{Narrow} & \multicolumn{2}{c}{Medium maze} & \multicolumn{2}{c}{AntMaze} \\
\cmidrule(lr){2-3} \cmidrule(lr){4-5} \cmidrule(lr){6-7} \cmidrule(lr){8-9}
Method & Median & Mean & Median & Mean & Median & Mean & Median & Mean \\
\midrule
Unguided & 121 & 120 & 107 & 108 & 86 & 86 & 213 & 213 \\
RDM + projection & 4.8 & 41 & 5.5 & 6.4 & 44 & 51 & 653 & 621 \\
Projection & 142 & 149 & 27\,238 & 26\,143 & 101 & 101 & 214 & 236 \\
DiRecT & 100 & 101 & 467 & 688 & 139 & 138 & 399 & 2\,048 \\
SafeDiffuser & 235 & 236 & 328 & 331 & 210 & 210 & 493 & 486 \\
\midrule
Best-of-$N$ & 120 & 121 & 109 & 109 & 82 & 82 & 200 & 203 \\
SMC & 112 & 112 & 118 & 118 & 86 & 87 & 224 & 225 \\
Persistent best-of-$N$ & 108 & 109 & 89 & 90 & 22 & 22 & 171 & 174 \\
P3 (ours) & 117 & 118 & 93 & 93 & 24 & 24 & 186 & 187 \\
P3 without rejuvenation & 65 & 65 & -- & -- & -- & -- & -- & -- \\
\bottomrule
\end{tabular}

\end{table}

\section{Ablations}
\label{app:ablations}

\newcommand{\ablPopCold}{%
        $N=1$
        & \result{0.75}{0.11}
        & \result{0.00}{0.00}
        & \result{2.3805}{0.5500}
        & \result{7.46}{3.34}
        & 20
        & 0.060 \\
        $N=2$
        & \result{0.62}{0.07}
        & \result{0.02}{0.02}
        & \result{0.8752}{0.2505}
        & \result{11.22}{4.76}
        & 40
        & 0.117 \\
        $N=4$
        & \result{0.77}{0.18}
        & \result{0.16}{0.20}
        & \result{0.1440}{0.1234}
        & \result{12.92}{4.03}
        & 80
        & 0.114 \\
        $N=8$
        & \result{0.87}{0.10}
        & \result{0.31}{0.31}
        & \result{0.0445}{0.0400}
        & \result{12.88}{4.71}
        & 160
        & 0.105 \\
        $N=16$
        & \result{0.90}{0.08}
        & \result{0.35}{0.33}
        & \result{0.0323}{0.0288}
        & \result{11.42}{5.21}
        & 320
        & 0.115 \\
        $N=32$
        & \result{0.97}{0.02}
        & \result{0.46}{0.42}
        & \result{0.0212}{0.0200}
        & \result{8.68}{3.09}
        & 640
        & 0.134 \\
}
\newcommand{\ablPopPthree}{%
        $N=1$
        & \result{0.65}{0.18}
        & \result{0.06}{0.04}
        & \result{1.5026}{0.9990}
        & \result{3.08}{1.20}
        & 11
        & 0.034 \\
        $N=2$
        & \result{0.69}{0.19}
        & \result{0.29}{0.18}
        & \result{0.6528}{0.7916}
        & \result{4.64}{1.32}
        & 22
        & 0.067 \\
        $N=4$
        & \result{0.94}{0.06}
        & \result{0.76}{0.14}
        & \result{0.0413}{0.0203}
        & \result{4.56}{0.45}
        & 44
        & 0.064 \\
        $N=8$
        & \result{0.99}{0.02}
        & \result{0.92}{0.07}
        & \result{0.0002}{0.0002}
        & \result{4.52}{0.58}
        & 88
        & 0.060 \\
        $N=16$
        & \result{1.00}{0.01}
        & \result{0.99}{0.02}
        & \result{0.0035}{0.0070}
        & \result{4.46}{0.59}
        & 176
        & 0.067 \\
        $N=32$
        & \result{1.00}{0.00}
        & \result{0.99}{0.02}
        & \result{0.0002}{0.0002}
        & \result{4.24}{0.92}
        & 352
        & 0.076 \\
}
\newcommand{\ablTzero}{%
        SMC
        & \result{0.90}{0.08}
        & \result{0.35}{0.33}
        & \result{0.0323}{0.0288}
        & \result{11.42}{5.21}
        & 20
        & 0.114 \\
        $t_0=2$
        & \result{0.96}{0.03}
        & \result{0.84}{0.13}
        & \result{0.1201}{0.1122}
        & \result{3.18}{0.56}
        & 3
        & 0.020 \\
        $t_0=4$
        & \result{0.99}{0.01}
        & \result{0.98}{0.02}
        & \result{0.0106}{0.0211}
        & \result{3.82}{0.48}
        & 5
        & 0.031 \\
        $t_0=6$
        & \result{1.00}{0.00}
        & \result{0.98}{0.01}
        & \result{0.0001}{0.0001}
        & \result{3.80}{0.41}
        & 7
        & 0.043 \\
        $t_0=8$
        & \result{1.00}{0.00}
        & \result{0.99}{0.01}
        & \result{0.0000}{0.0000}
        & \result{4.12}{0.57}
        & 9
        & 0.053 \\
        $t_0=10$
        & \result{1.00}{0.01}
        & \result{0.99}{0.02}
        & \result{0.0035}{0.0070}
        & \result{4.46}{0.59}
        & 11
        & 0.063 \\
        $t_0=12$
        & \result{1.00}{0.01}
        & \result{0.97}{0.02}
        & \result{0.0001}{0.0002}
        & \result{4.86}{0.85}
        & 13
        & 0.075 \\
        $t_0=14$
        & \result{0.99}{0.02}
        & \result{0.90}{0.14}
        & \result{0.0009}{0.0015}
        & \result{5.20}{0.82}
        & 15
        & 0.085 \\
        $t_0=16$
        & \result{0.98}{0.02}
        & \result{0.82}{0.17}
        & \result{0.0022}{0.0027}
        & \result{5.62}{0.60}
        & 17
        & 0.098 \\
        $t_0=19$
        & \result{0.96}{0.03}
        & \result{0.40}{0.40}
        & \result{0.0305}{0.0294}
        & \result{11.36}{5.06}
        & 20
        & 0.113 \\
}
\newcommand{\ablCarry}{%
        $t_0=2$
        & \result{0.84}{0.13}
        & \result{0.1201}{0.1122}
        & \result{3.18}{0.56} \\
        $t_0=4$
        & \result{0.98}{0.02}
        & \result{0.0106}{0.0211}
        & \result{3.82}{0.48} \\
        $t_0=6$
        & \result{0.98}{0.01}
        & \result{0.0001}{0.0001}
        & \result{3.80}{0.41} \\
        $t_0=8$
        & \result{0.99}{0.01}
        & \result{0.0000}{0.0000}
        & \result{4.12}{0.57} \\
        $t_0=10$
        & \result{0.99}{0.02}
        & \result{0.0035}{0.0070}
        & \result{4.46}{0.59} \\
        $t_0=12$
        & \result{0.97}{0.02}
        & \result{0.0001}{0.0002}
        & \result{4.86}{0.85} \\
        $t_0=14$
        & \result{0.90}{0.14}
        & \result{0.0009}{0.0015}
        & \result{5.20}{0.82} \\
        $t_0=16$
        & \result{0.82}{0.17}
        & \result{0.0022}{0.0027}
        & \result{5.62}{0.60} \\
        $t_0=19$
        & \result{0.40}{0.40}
        & \result{0.0305}{0.0294}
        & \result{11.36}{5.06} \\
}
\newcommand{\ablReset}{%
        $t_0=2$
        & \result{0.84}{0.13}
        & \result{0.2293}{0.2035}
        & \result{3.56}{0.61} \\
        $t_0=4$
        & \result{0.94}{0.06}
        & \result{0.0559}{0.0888}
        & \result{4.04}{0.27} \\
        $t_0=6$
        & \result{0.98}{0.01}
        & \result{0.0079}{0.0158}
        & \result{4.12}{0.31} \\
        $t_0=8$
        & \result{0.97}{0.02}
        & \result{0.0002}{0.0002}
        & \result{4.58}{0.48} \\
        $t_0=10$
        & \result{0.98}{0.03}
        & \result{0.0037}{0.0072}
        & \result{4.76}{0.62} \\
        $t_0=12$
        & \result{0.98}{0.01}
        & \result{0.0034}{0.0066}
        & \result{4.90}{0.52} \\
        $t_0=14$
        & \result{0.91}{0.12}
        & \result{0.0009}{0.0013}
        & \result{5.38}{0.56} \\
        $t_0=16$
        & \result{0.82}{0.22}
        & \result{0.0028}{0.0040}
        & \result{5.66}{0.70} \\
        $t_0=19$
        & \result{0.43}{0.41}
        & \result{0.0295}{0.0269}
        & \result{10.94}{5.23} \\
}

The population-size, re-noising-depth and weight-retention sweeps
use cluttered, five model seeds, 50 trials per seed, at most 200
control steps, the unchanged learned transition,
accumulated weighting with ESS resampling, and final selection by
minimum violation.
Along each swept axis the remaining settings are held at $N=16$
particles, re-noising level $t_0=10$, weights carried across control
steps, virtual potential margin $0.01$, and no rejuvenation.
Entries are means $\pm$ standard deviations across the five seeds.
Appendix~\ref{app:pbon} separately compares persistence with and
without SMC across all tasks, with its own trial counts.

\subsection{Population size}
\label{app:abl-pop}

\begin{table}[h]
    \centering
    \small
    \caption{
        Population size on cluttered constraints, at matched $N$ for
        SMC and P3 ($t_0=10$).
        Violation is the accumulated residual along the executed path
        and switches are plan-side changes per episode; lower is
        preferable for both.
        Evals is $N$ times the number of reverse iterations: per
        planning call for SMC and per warm replanning call for P3.
        P3's first call uses the full $20$ iterations.
        s/step is wall-clock measured under the protocol in
        Appendix~\ref{app:protocol}, in a separate run from
        Table~\ref{tab:abl-t0}.
        Denoiser evaluations grow linearly in $N$, whereas measured
        latency changes little between $N=2$ and $N=16$.
    }
    \label{tab:abl-pop}
    \begin{tabular}{@{}lcccccc@{}}
        \toprule
        & Task succ.\ $\uparrow$ & Clean succ.\ $\uparrow$
        & Violation $\downarrow$ & Switches $\downarrow$ & Evals $\downarrow$
        & s/step \\
        \midrule
        \multicolumn{7}{@{}l}{\textit{SMC}} \\
        \addlinespace[2pt]
        \ablPopCold
        \midrule
        \multicolumn{7}{@{}l}{\textit{P3 (ours), $t_0=10$}} \\
        \addlinespace[2pt]
        \ablPopPthree
        \bottomrule
    \end{tabular}
\end{table}

We vary $N$ for fresh and retained populations under the same
constraint potential.
At $N=1$, clean success is $0.00$ for cold SMC and $0.06$ for P3.
Proposition~\ref{prop:target}(a) shows that, with one particle,
constraint weighting cannot change the executed plan.
With two particles the effective sample size never falls below $N/2$,
so resampling never fires and the constraints act only through the
final choice between two unweighted lineages.
P3 at $N=8$ reaches $0.92$ clean success against $0.46$ for cold SMC
at $N=32$, using $8\times11=88$ denoiser evaluations per warm replanning call,
compared with $32\times20=640$ for cold SMC at $N=32$, and takes
$60$ rather than $134$\,ms in the serial CPU measurement.
The reduction is $7.3$-fold in denoiser evaluations and $2.2$-fold in
measured latency.
Proposition~\ref{prop:keep} compares finding acceptable plans with
retaining them under
independent final plans. Its $N\nu$ bound also holds for dependent
plans, but $\nu$ must then be their mean final-particle acceptance
rate, not the unweighted model's rate.
The $N=1$ result concerns this unchanged learned transition and
does not imply that all single-trajectory
warm-start methods behave the same way.
The sweep shows no consistent trend in switching with $N$ relative
to the between-seed variability.

\subsection{Re-noising depth}
\label{app:abl-t0}

\begin{table}[h]
    \centering
    \small
    \caption{
        Re-noising depth on cluttered constraints at $N=16$ with
        carried weights.
        The SMC row is the cold-start control.
        $t_0=19=K-1$ re-noises to the start of the reverse chain.
        The accumulated weights are retained, but the re-noised input
        depends only weakly on the previous trajectory.
        Iter.\ counts reverse diffusion iterations per warm replanning
        call for P3 and per planning call for SMC; P3's first call uses
        all $20$ iterations.
        Latency uses the separate serial CPU protocol of
        Appendix~\ref{app:protocol}. Cold SMC and $t_0=19$ have the same
        number of denoising iterations and latencies within $1\%$.
    }
    \label{tab:abl-t0}
    \begin{tabular}{@{}lcccccc@{}}
        \toprule
        & Task succ.\ $\uparrow$ & Clean succ.\ $\uparrow$
        & Violation $\downarrow$ & Switches $\downarrow$ & Iter. & s/step \\
        \midrule
        \ablTzero
        \bottomrule
    \end{tabular}
\end{table}

Clean success is highest at intermediate depths, $0.97$--$0.99$
for $t_0=4$--$12$, compared with $0.84$ at $t_0=2$ and $0.40$ at
$t_0=19$. The shallowest setting also has the largest accumulated
violation among the persistent configurations. At $t_0=19$, the
re-noised input depends only weakly on the retained plan
(Appendix~\ref{app:proof-target}).
Plan-side switching increases overall, from
$3.18$ at $t_0=2$ to $11.36$ at $t_0=19$, against $11.42$ for cold SMC.
This agrees qualitatively with the rise toward the fresh-draw rate
in Theorem~\ref{thm:memory}; the theorem does not predict the
closed-loop values because the implemented repair is inexact
(Appendix~\ref{app:kernel}).
At $t_0=6$ and $t_0=10$, both within the high-success range, serial
latency on one idle CPU is $43$ and $63$\,ms per step, compared with
$114$\,ms for cold SMC.

\subsection{Weight retention}
\label{app:abl-weights}

\begin{table}[h]
    \centering
    \small
    \caption{
        Resetting weights to uniform at every control step, at $N=16$
        on cluttered. The carried-weight comparison is in
        Table~\ref{tab:abl-t0}. Both variants retain trajectories and
        differ only in the weights entering the next step.
    }
    \label{tab:abl-weights}
    \begin{tabular}{@{}lccc@{}}
        \toprule
        & Clean succ.\ $\uparrow$ & Violation $\downarrow$
        & Switches $\downarrow$ \\
        \midrule
        \ablReset
        \bottomrule
    \end{tabular}
\end{table}

Table~\ref{tab:abl-weights} reports reset weights, to compare with
the carried weights in Table~\ref{tab:abl-t0}.
At $t_0=10$, carrying versus resetting weights gives mean clean
success of $0.99$ versus $0.98$. Carrying weights reduces accumulated
violation at shallow depths, including $0.1201$ versus $0.2293$
at $t_0=2$.
Carried weights also reduce plan-side switching slightly at every depth
below $t_0=19$.

\paragraph{Carried weights without trajectory memory.}
At $t_0=19$, the re-noising reaches the start of the reverse chain,
where $\snr_{t_0}\approx0.005$. The re-noised input therefore depends
only weakly on the previous trajectory, while the accumulated weights
are retained.
Cold SMC, carried weights and reset weights have mean clean success
$0.35$, $0.40$ and $0.43$, respectively, and similar observed switch rates.
These clean-success rates are below the $0.97$--$0.99$ observed at
$t_0=4$--$12$: retaining weights at $t_0=19$ does not reproduce those
intermediate-depth results.
The near-equality of the three switch rates is consistent with the
fresh-draw limit of Theorem~\ref{thm:memory}(ii).

\subsection{Persistence without SMC weighting and resampling}
\label{app:pbon}

\begin{table}[t]
    \centering
    \footnotesize
    \caption{
        Persistent best-of-$N$ against P3 at the ablation setting ($N=16$,
        $t_0=10$, margin $0.01$, no rejuvenation), five model seeds with
        100 trials each.
    }
    \label{tab:pbon-abl}
\setlength{\tabcolsep}{4pt}
\begin{tabular}{@{}llcccc@{}}
\toprule
 & Method & Both-hard & Cluttered & Left-hard & Right-hard \\
\midrule
Clean & Persistent best-of-$N$ & \result{0.54}{0.30} & \result{0.87}{0.18} & \result{0.46}{0.12} & \result{0.77}{0.18} \\
 & P3 & \result{0.78}{0.13} & \result{0.99}{0.01} & \result{0.81}{0.14} & \result{0.95}{0.05} \\
\addlinespace[2pt]
Violation & Persistent best-of-$N$ & \result{0.6503}{0.8046} & \result{0.0009}{0.0012} & \result{0.0245}{0.0113} & \result{0.0128}{0.0216} \\
 & P3 & \result{0.2855}{0.3524} & \result{0.0045}{0.0088} & \result{0.0065}{0.0059} & \result{0.0004}{0.0005} \\
\addlinespace[2pt]
Switches & Persistent best-of-$N$ & \result{1.48}{0.53} & \result{4.80}{0.70} & \result{3.55}{1.41} & \result{4.54}{1.12} \\
 & P3 & \result{1.39}{0.44} & \result{4.44}{0.54} & \result{3.68}{1.03} & \result{4.45}{0.59} \\
\addlinespace[2pt]
\bottomrule
\end{tabular}

\end{table}

Resetting the carried weights leaves the weighting and resampling
inside every planning call in place, so it cannot show whether this
selection still matters once plans are kept.
Persistent best-of-$N$ removes it: at every control step all $N$ plans
are shifted and re-noised as in Section~\ref{sec:persistence}, denoised
with the unweighted reverse kernel, and the plan with the smallest
violation is executed.
Each column runs it at that column's P3 setting: on Avoiding $N=32$,
$t_0=10$, potential margin $0.04$ and eight fresh draws per step, which
replace randomly chosen plans because persistent best-of-$N$ has no weights to rank
them by; on the large maze $N=16$ and $t_0=16$; on the medium maze
$N=16$, $t_0=4$ and margin $0.3$; on AntMaze $N=4$ and $t_0=40$ of
$K=50$, shifted by the five executed steps.
Every comparison pairs methods by trial seed, with the same initial
environment state.

Tables~\ref{tab:results}, \ref{tab:task-success} and~\ref{tab:full-avoiding}
report clean success, task success and violation for P3 and persistent
best-of-$N$, which uses neither SMC weighting nor resampling.
On cluttered, persistent best-of-$N$ is clean on $0.87$ of trials, compared with
$0.34$ for fresh best-of-$N$ and $0.48$ for cold SMC.
On the medium maze, persistent best-of-$N$ and P3 have mean clean
success $0.20$ and $0.21$; fresh populations stay at or below
$0.01$. P3 has higher observed clean success than persistent
best-of-$N$ on the other tasks, with the largest large-maze difference
on narrow ($0.85$ against $0.56$).
Without rejuvenation and at the smaller population of the ablations
(Table~\ref{tab:pbon-abl}), the clean-success difference is larger on
both-hard, left-hard and right-hard. On both-hard and left-hard, persistent best-of-$N$ reaches $0.54$ and $0.46$, compared with $0.78$ and $0.81$ for P3.
Its observed plan-side switch rates remain close to P3's.

\section{Large maze}
\label{app:maze}

The main text reports both presets at $H=64$; this appendix gives their
protocol and further results.

\paragraph{Configurations.}
The large maze is the D4RL maze2d-large layout of Gymnasium-Robotics
($12\times 9$ cells), with the start cell $(1,3)$ and the goal cell
$(7,9)$ fixed.
The \emph{broad} and \emph{narrow} presets are the superellipse obstacle
configurations that DiRecT \citep{direct2026} defines for this layout,
with two and six obstacles (Figure~\ref{fig:envs}); the walls are
enforced by the simulator and are not part of the constraint set.
Episodes last at most $1200$ control steps.
The sampling methods pad the obstacles by $0.01$ in their potentials and
selectors, and Projection by the ball radius $0.1$.

\paragraph{Planning model.}
The maze model is a joint state--action diffusion model over $64$-step
segments (forces, positions, velocities and the goal position pinned at
every step), a temporal U-Net of base width $64$ trained for $60$k steps
on $250$ self-collected episodes of the large maze with random start and
goal cells; five independently initialised models are used.
Sampling clips the predicted clean plan to the normaliser's range at
every reverse step.
Unguided, the models reach the fixed goal cell on every trial, but their
routes graze the test-time obstacles: clean success is $0.06$ on broad
and $0.02$ on narrow (Table~\ref{tab:results}).

\paragraph{Potential.}
The maze potential is the summed penetration depth of the plan's
positions after the pinned start ($63$ at $H=64$) into the test-time obstacles, padded by $0.01$.
Unlike Avoiding's maximum over geometric, dynamics and bound residuals,
this score measures only obstacle penetration. All population methods
use the same score; dynamics rows remain in the projector's constraint set.
SMC and P3 use $\lambda_{\mathrm{end}}=\beta_{\mathrm{end}}=5$.

\begin{figure}[t]
    \centering
    \includegraphics[width=\linewidth]{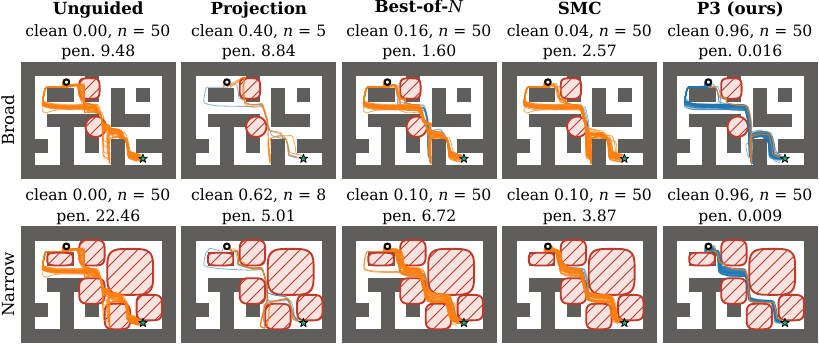}
    \caption{
        Executed paths on the large maze for one model seed, broad (top)
        and narrow (bottom). Blue: goal reached cleanly;
        orange: after a violation. Titles give clean success, mean
        penetration and trial count $n$.
    }
    \label{fig:maze-paths}
\end{figure}

\paragraph{Executed paths.}
Figure~\ref{fig:maze-paths} overlays the executed paths of every trial
for one model seed.
All routes are in the same homotopy class, so the constraint concerns
where inside the corridor the ball travels. The plotted P3 trajectories
pass farther from the obstacles at corridor corners.

\paragraph{Projection baseline.}
Projection projects the predicted clean plan onto the padded constraint set
at each of the last $11$ of the $20$ reverse steps, with SLSQP over $512$
variables.
On one GPU it takes $27$\,s per control step on narrow, against
$0.09$--$0.12$\,s for the population methods (Appendix~\ref{app:timing}).
Its computational cost limited the campaign to $5$--$10$ trials per seed.

\paragraph{Blocked-corridor preset.}
A further preset replaced the broad route obstacle by a block that closes
the direct corridor for the ball, leaving a clean route of $13.8$ maze
units against $9.0$ for the direct one.
Every sampling method, P3 included, drove through the block on all 25
trials of one model seed. The model was trained on shortest-path
demonstrations with random endpoints, and none of the methods found
the detour with $N=16$ in these trials.
These trials do not establish that the detour has zero model probability.
A population of $N$ fresh plans finds a route of prior mass $p$ with
probability at most $Np$, so low-probability routes can remain undiscovered.
The medium-maze experiment of Appendix~\ref{app:mb} instead trains on
demonstrations of both routes and uses horizon $H=8$.

\section{Medium maze}
\label{app:mb}

\paragraph{Configuration.}
The medium maze is the D4RL maze2d-medium layout of Gymnasium-Robotics
($8\times 8$ cells).
Between the start cell $(1,1)$ and the goal cell $(6,6)$ the layout
offers two routes: a direct corridor through row~3 of ten cells, and a
detour through the lower-left rooms of thirteen cells.
The \emph{blocked} preset places a power-4 superellipse of semi-axis
$0.6$ on the first cell of the direct corridor, which closes it for the
ball, and a disc of radius $0.6$ over the row-4 corridor before the goal,
which leaves a passage $0.2$ units wide along the corridor's lower wall
(Figure~\ref{fig:envs}).
The preset is ours; DiRecT publishes obstacles for the large layout only.
As on the large maze the walls are enforced by the simulator and are not
part of the constraint set.

\paragraph{Data and model.}
The planning model is trained on $200$ scripted demonstrations of this
single start--goal pair, $100$ through each route, generated by the
same waypoint-following PD controller and action noise as the large-maze
data (Figure~\ref{fig:mb-data}, left).
The model is the Avoiding architecture at $H=8$, $K=20$, base width
$32$, trained for $30$k steps. We omit goal conditioning because every
demonstration and evaluation episode uses the same goal.
Five independently initialised models are used.

\paragraph{Behavior of the unguided model.}
Half the demonstrations use the detour, but only about $1\%$ of the
unguided per-step rollouts enter it across the five evaluation seeds.
Figure~\ref{fig:mb-data} (right) shows rollouts that follow the direct
corridor.
An $8$-step plan covers about $0.07$ maze units.
At the branch cell, the sampled continuation depends strongly on the
conditioning velocity: at rest, $99\%$ of $256$ sampled plans head
toward the direct corridor; with the demonstrations' downward entry
velocity, $100\%$ head toward the detour.
P3 enters the detour in $42\%$ of the evaluation episodes.

\paragraph{Horizon.}
Table~\ref{tab:maze-results} uses $H=8$.
At this horizon, projected plans often stop in front of the block, and
the agent can remain there across repeated replanning calls.

\paragraph{Protocol and settings.}
$50$ trials per seed and method, replanning after every action, at most
$800$ control steps, with methods paired by trial seed.
Unguided sampling uses one particle; Projection pads the obstacles by the
ball radius $0.1$; best-of-$N$, cold SMC and P3 use $N=16$,
$\lambda=\beta=5$ and pad the obstacles by $0.3$ in the potential and
the selector.
P3 uses $t_0=4$.


\paragraph{Results.}
Projection enters the detour in $0.59$ of episodes, compared with $0.42$
for P3, but reaches the goal in only $0.26$.
Its mean accumulated penetration is $4.0$, compared with $7.1$ for P3.
Figure~\ref{fig:route} (right) includes trajectories that stop near the
block, enter the blocked corridor, or take the detour.

\section{AntMaze}
\label{app:antmaze}

\begin{figure}[!t]
    \centering
    \includegraphics{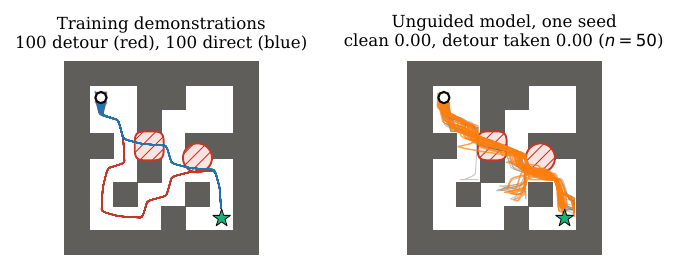}
    \caption{
        Medium maze, blocked. Left: the $200$ training demonstrations, half
        through the detour (red) and half through the direct corridor (blue);
        the obstacles exist only at test time. Right: the unguided $H=8$ model
        takes the direct corridor on every trial in the plotted seed (orange: goal
        after a violation; grey: failure).
    }
    \label{fig:mb-data}
    \vspace{10pt}
    \includegraphics{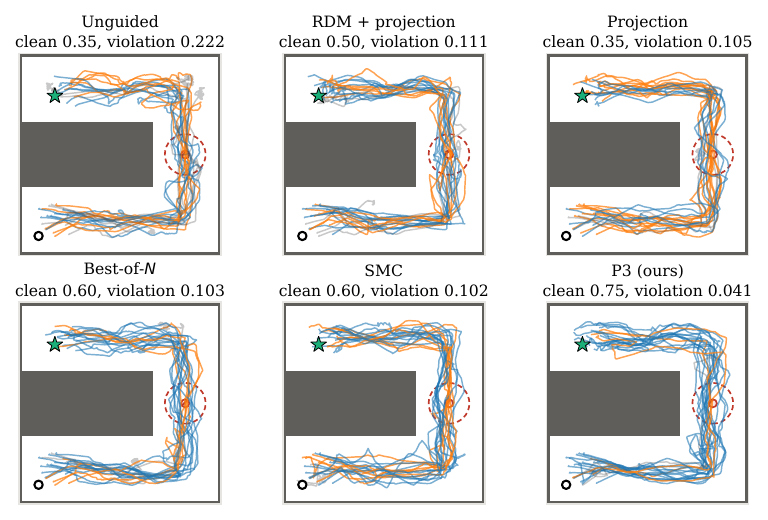}
    \caption{
        Executed AntMaze paths for one block of $20$ trials, population
        methods at $N=32$. Pooled results over all six blocks are in
        Table~\ref{tab:results} for the one-sample methods and in
        Table~\ref{tab:antmaze-n} for SMC and P3.
        Blue: clean success; orange: goal after
        entering the obstacle; grey: failure. Hatched disc: graded region;
        dashed circle: planning clearance.
    }
    \label{fig:antmaze-paths}
\end{figure}

\paragraph{Model and closed-loop protocol.}
We use the $H=32$, $K=50$ goal-conditioned inverse-dynamics checkpoint
trained on \texttt{antmaze-umaze-v1}.  The diffusion model predicts a
$29$-dimensional state trajectory; a two-hidden-layer inverse-dynamics
network with Mish activations maps adjacent predicted states to the
eight Ant torques.  We load the EMA weights, do not clip denoised
states, execute five actions from each plan, and then replan from the
measured state.  Unguided sampling, Projection, best-of-$N$ and SMC begin
each replan from fresh Gaussian noise, and RDM + projection keeps one
plan (Appendix~\ref{app:rdm}).  P3 carries its population and log-weights
across replans: it drops the five executed states of every plan, repeats the
last one to restore the horizon, and re-noises to $t_0=40$, where
$\snr_{t_0}=0.58$.

\paragraph{Constraint and methods.}
The inference-only obstacle is a disc centered at $(4,0)$ in the
vertical corridor.  Its physical radius $0.25$ grades the executed
center path, while every constraint-aware method plans with an
additional $1.0$ clearance margin (Figure~\ref{fig:envs}).  It is
absent from both the training data and MuJoCo collision geometry.
Unguided uses one constraint-unaware plan; RDM + projection keeps one
plan, using the same projector and margins as Projection, and repairs
or regenerates it when the model's likelihood drops or the plan enters
the padded disc. Projection projects
one sample during the last $10\%$ of the reverse chain.  In
Table~\ref{tab:results}, best-of-$N$, SMC, persistent best-of-$N$ and P3
use $N=4$; best-of-$N$ executes the least-violating of $N$ independent
plans.  SMC and P3 use $\lambda_{\rm end}=\beta_{\rm end}=10$,
accumulated weighting, ESS resampling and minimum-violation final
selection, and differ only in P3's persistence; P3 uses no
rejuvenation.

\paragraph{Population size.}
Table~\ref{tab:antmaze-n} gives SMC and P3 at $N=4$, $8$, $16$ and $32$
on the same $120$ episodes.
P3's clean success stays between $0.57$ and $0.62$ at every $N$.
SMC is lower at $N=4$ and $8$ than at $N=32$.
With $N=32$, P3 and SMC are close ($0.61$ and $0.60$), and P3 enters the
disc in $23\%$ of episodes against $30\%$ for SMC.

\begin{table}[h]
    \centering
    \footnotesize
    \caption{AntMaze clean success of SMC and P3 by population
        size $N$, mean $\pm$ standard deviation across six evaluation seeds
        of $20$ trials. Table~\ref{tab:results} reports $N=4$.
    }
    \label{tab:antmaze-n}
\begin{tabular}{@{}rcc@{}}
\toprule
$N$ & SMC & P3 (ours) \\
\midrule
4 & \result{0.53}{0.11} & \result{0.62}{0.10} \\
8 & \result{0.50}{0.12} & \result{0.57}{0.12} \\
16 & \result{0.58}{0.09} & \result{0.57}{0.10} \\
32 & \result{0.60}{0.08} & \result{0.61}{0.10} \\
\bottomrule
\end{tabular}

\end{table}

\paragraph{Tracking and constraint violations.}
Around the turn, the executed state trajectory can differ from the
projected, geometrically feasible plan.
Population sampling selects among model-generated trajectories without
projecting their states, but this does not guarantee clean execution
either (Figure~\ref{fig:antmaze-paths}).

\end{document}